\documentclass{article}

\usepackage{arxiv}
\usepackage[utf8]{inputenc} % allow utf-8 input
\usepackage[T1]{fontenc}    % use 8-bit T1 fonts
\usepackage{url}            % simple URL typesetting
\usepackage{booktabs}       % professional-quality tables
\usepackage{nicefrac}       % compact symbols for 1/2, etc.
\usepackage{microtype} 

\usepackage{microtype}
\usepackage{graphicx}
\usepackage{subfigure}
\usepackage{booktabs} % for professional tables
\usepackage{array}

\usepackage{hyperref}

\usepackage{caption}
\usepackage{subfigure}
\usepackage{subfloat}
\usepackage{float}

\usepackage{amsmath}
\usepackage{amssymb}
\usepackage{mathtools}
\usepackage{amsthm}
\usepackage{tikz}
\usetikzlibrary{cd}

\usepackage[capitalize,noabbrev]{cleveref}

\theoremstyle{plain}
\newtheorem{theorem}{Theorem}[section]
\newtheorem{proposition}[theorem]{Proposition}

\newtheorem{corollary}[theorem]{Corollary}
\theoremstyle{definition}
\newtheorem{definition}[theorem]{Definition}

\theoremstyle{remark}

\title{Topology-Preserving Dimensionality Reduction via Interleaving Optimization}
\DeclareMathOperator{\img}{img}
\DeclareMathOperator{\dgm}{dgm}
\newcommand{\calR}{\mathcal{R}}
\newcommand{\dB}{\mathrm{d}_{\mathrm{B}}}
\newcommand{\dI}{\mathrm{d}_{\mathrm{I}}}

\author{
Bradley J. Nelson\thanks{Two authors contributed equally to this work.} \thanks{Corresponding author} \\
	Department of Statistics\\
	University of Chicago\\
	Chicago, IL 60637\\
	\texttt{bradnelson@uchicago.edu} \\
   \And
 Yuan Luo\footnotemark[1] \footnotemark[2]\\
	Committee on Computational and Applied Mathematics\\
	University of Chicago\\
	Chicago, IL 60637\\
	\texttt{yuanluo@uchicago.edu}\\
}

\begin{document}
\maketitle
\begin{abstract}
Dimensionality reduction techniques are powerful tools for data preprocessing and visualization which typically come with few guarantees concerning the topological correctness of an embedding.  The interleaving distance between the persistent homology of Vietoris-Rips filtrations can be used to identify a scale at which topological features such as clusters or holes in an embedding and original data set are in correspondence.  We show how optimization seeking to minimize the interleaving distance can be incorporated into dimensionality reduction algorithms, and explicitly demonstrate its use in finding an optimal linear projection. We demonstrate the utility of this framework to data visualization. 
\end{abstract}


\section{Introduction}

Dimension reduction is an important component of many data analysis tasks, but can be potentially problematic as it may ``reveal'' structure in data which is not truly present.  In inference this can be addressed by principled use of a withheld test set or an analysis which addresses model selection more directly.  
However, in exploratory data analysis it can be difficult to address selection problems incurred by exploration of different dimension reduction techniques, such as whether visualized structures are really present or an artifact of the chosen embedding.
%and subsequent parameter tuning can generally cast doubt on whether visualized structures are really present or an artifact of the chosen embedding.  
In this paper, we develop the use of the interleaving distance for the purpose of quantifying the extent to which \emph{topological} features of an embedding relate to features in the original data set.  Explicitly, we can compute a threshold after which features of a certain size in the persistent homology of the Vietoris-Rips filtration are in one-to-one correspondence between the data set before and after dimension reduction.  Furthermore, we show how to find local minima of this threshold through optimization and demonstrate this on the task of finding optimal projections of a data set.

\subsection{Related Work}
%Brad: I'm following the scheme of some other papers, in ICML such as topological autoencodes which discuss related work after the technical content.  I think this makes sense.

Optimization of persistent homology-based objective functions has attracted much recent attention \cite{poulenardTopologicalFunctionOptimization2018,gabrielssonTopologyLayerMachine2020,hoferGraphFiltrationLearning2020, kimPLLayEfficientTopological2020, carrierePersLayNeuralNetwork2020a, carriereOptimizingPersistentHomology2021, leygonieFrameworkDifferentialCalculus2021} with a particular focus on applications in computational geometry and deep learning.  Of particular relevance is the work of \cite{moorTopologicalAutoencoders2020} which uses a persistence-based objective to preserve critical edges in the persistent homology of the Vietoris-Rips filtration in a learned latent space of an autoencoder.

The idea of using persistent homology to compare different dimension reduction schemes was initiated by Rieck and Leitte \cite{rieckPersistentHomologyEvaluation2015, rieckAgreementAnalysisQuality2017}, which uses the $2$-Wasserstein distance to compare the persistence diagrams of a data set before and after dimension reduction.  Several non-differentiable methods incorporating persistent homology have been developed \cite{desilvaPersistentCohomologyCircular2011a,yanHomologyPreservingDimensionalityReduction2018, doraiswamyTopoMap0dimensionalHomology2021}.  With the development of optimization techniques for persistent homology, several differentiable methods have been proposed based on optimization of the $2$-Wasserstein metric on persistence diagrams \cite{kachanPersistentHomologybasedProjection2020, wagnerImprovingMetricDimensionality2021} and an approach based on simulated annealing \cite{yuShapePreservingDimensionalityReduction2021}.  While not employed on Vietoris-Rips filtrations, the work of \cite{poulenardTopologicalFunctionOptimization2018} uses the bottleneck distance for optimization of functional maps on shapes.

The interleaving/bottleneck distance has long been a key tool developed for the study of persistent homology under perturbation of the input \cite{cohen-steinerStabilityPersistenceDiagrams2007}.  The interleaving distance was first introduced by \cite{chazalProximityPersistenceModules2009}, and applied to the study of Vietoris-Rips filtrations in \cite{chazalGromovHausdorffStableSignatures2009} to bound the distance of persistent homology by the Gromov-Hausdorff distance between the input point clouds.  The idea of developing confidence regions for persistence pairs was developed by \cite{fasyConfidenceSetsPersistence2014} in the context of sampling.

\subsection{Contributions}

This work presents a novel approach to dimension reduction using optimization of the interleaving/bottleneck distance between the persistent homology of Vietoris-Rips Filtrations of an original data set $X$ and the data set $Y$ after dimensionality reduction.
\begin{enumerate}
\setlength\itemsep{0em}
    \item We show how the interleaving distance can be used to quantify a scale at which topological features in $X$ and features in $Y$ are in correspondence, and be used to select homological features of $Y$ in correspondence with features in $X$.
    \item We show how to incorporate the interleaving distance explicitly into the optimization of the embedding $Y$ and prove the existence of descent directions under mild conditions.
    \item We demonstrate this technique in finding optimal linear projections of the data set $X$ to preserve the bottleneck distance on several examples\footnote{Our implementations are made publicly available at 
\url{https://github.com/CompTop/Interleaving-DR}.} with interesting topology.
\end{enumerate}

\section{Background}\label{section:Background}

\subsection{Persistent Homology of Vietoris-Rips Filtrations}

We are interested in discovering and preserving topological features of a point cloud $X$ together with a notion of dissimilarity $d$, and refer to the combination of these two data as a dissimilarity space $(X,d)$.  A dissimilarity is a function $d:X\times X\to \mathbb{R}_{\ge 0}$, with $d(x,x)=0$ for any $x\in X$.  We will typically consider dissimilarities that are metrics (in particular which satisfy triangle inequality), but many of the bounds here hold more generality. 
% \textcolor{red}{Not necessary/relevant sentence or should be divided into a separate paragraph}
Examples of topological features of the space $(X,d)$ include clusters formed through single-linkage clustering or ``holes'' forming loops in the $r$-nearest neighbors graph of $X$.

Vietoris-Rips filtrations are commonly used in conjunction with persistent homology to create features for finite dimensional metric spaces (point clouds) \cite{carlssonTopologicalPatternRecognition2014b}. Given a dissimilarity space $(X, d)$, a Vietoris-Rips complex consists of simplices with a maximum pairwise dissimilarity between vertices is less than some threshold $r$ :
$$
X_{r}=\left\{\left(x_{0}, \ldots, x_{k}\right) \mid x_{i} \in X, d\left(x_{i}, x_{j}\right) \leq r\right\}
$$
A Vietoris-Rips filtration is a nested sequence of Vietoris-Rips complexes $X_{r} \subseteq X_{s}$ if $r \leq s$.

Homology is a functor from the category of topological spaces and continuous maps to the category of vector spaces and linear maps (for a general introduction see \cite{hatcherAlgebraicTopology2002}).  The dimension of the $k$-dimensional homology vector space $H_k(X)$ of a topological space $X$ counts $k$-dimensional topological features of $X$: $\dim H_0(X)$ is the number of connected components, $\dim H_1(X)$ counts loops, and $\dim H_k$ generally counts $k$-dimensional voids.  A continuous map $f:X\to Y$ has an induced map $H_k(f):H_k(X) \to H_k(Y)$ which maps vectors associated with topological features in $X$ to vectors associated with topological features in $Y$.  The computation of $H_k(X)$ begins with the construction of a chain complex $C_\ast(X) = \{C_k(X), \partial_k: C_k(X) \to C_{k-1}(X)\}_{k\ge 0}$ where $C_k(X)$ is a vector space with a basis element for each $k$-simplex in $X$, and the boundary map $\partial_k$ sends each basis element to a linear combination of basis elements of faces in the boundary of the associated simplex.  
The boundary maps satisfy $\partial_k \circ \partial_{k+1} = 0$, and $H_k(X)$ is the quotient vector space $\ker \partial_k / \img \partial_{k+1}$.

Persistent homology \cite{edelsbrunnerTopologicalPersistenceSimplification2002, zomorodianComputingPersistentHomology2005} is an algebraic invariant of filtrations which captures how the topology of a filtration changes using homology.  The output of persistent homology is a persistence vector space $V_\ast$, consisting of vector spaces $\{V_r = H_k(X_r)\}_{r\in \mathbb{R}}$ and linear maps induced by inclusion $\{\iota^V_{r,s}:V_r\to V_s\}_{r \le s \in \mathbb{R}}$ which satisfy a consistency condition $\iota^V_{r,t} = \iota^V_{r,s} \iota^V_{s,t}$ for all $r \le s \le t$.  

Persistence vector spaces are classified up to isomorphism by birth-death pairs $\{(b_i, d_i)\}_{i\in I}$, or equivalently their persistence barcode \cite{zomorodianComputingPersistentHomology2005} or interval indecomposables \cite{carlssonZigzagPersistence2010}.  Each pair $(b,d)$ is associated to the appearance of a new homology vector at filtration parameter $b$ (meaning it is not in the image of an induced map), which maps through the persistence vector space until it enters the kernel of an induced map at filtration parameter $d$.  The length of the pair $(b,d)$ is the difference $|d- b|$.

Every birth and death in persistent homology is associated with the addition of a particular simplex in the filtration.  This follows from the definition of homology of the quotient vector space $H_k(X) = \ker \partial_k / \img \partial_{k+1}$.  The addition of a new $k$-simplex increases the dimension of $C_k(X)$ by one, and either increases the dimension of $\ker \partial_k$ by one, causing a birth in $H_k(X)$, or increases the dimension of $\img \partial_{k}$ by one, causing a death in $H_{k-1}(X)$. 
% \textcolor{red}{Have defined inverse map here } 
Because the Vietoris-Rips filtration is determined by its edges, the filtration value of every simplex can be mapped to the largest pairwise distance.  This provides a way to map the gradient of a function with respect to births and deaths to a gradient with respect to each pairwise distance -- see \cite{gabrielssonTopologyLayerMachine2020} for additional details.

% Persistence vector spaces, barcodes

% Filtrations, persistent homology

% Interpretation of birth-death pairs in terms of topological features. 

\subsection{The Interleaving and Bottleneck Distances}

Interleavings allow for the comparison of two persistence vector spaces \cite{chazalProximityPersistenceModules2009}, as well as other objects filtered by some partially ordered set \cite{bubenikCategorificationPersistentHomology2014}.  Let $V_\ast$ and $W_\ast$ be 1-parameter persistence vector spaces.  An $\epsilon$-shift map $f_\ast:V_\ast \to W_\ast$ is a collection of maps $f_r:V_r\to W_{r+\epsilon}$ so that the following diagram commutes for all parameters $r$
\begin{equation}
\begin{tikzcd}
V_r \ar[r,"\iota^V_{r,s}"]\ar[d,"f_r"] & V_s\ar[d, "f_s"]\\
W_{r+\epsilon} \ar[r,"\iota^W_{r+\epsilon, s+\epsilon}"] & W_{s+\epsilon}
\end{tikzcd}
\end{equation}
An $\epsilon$-interleaving between $V_\ast$ and $W_\ast$ is a pair of $\epsilon$-shift maps $f_\ast:V_\ast\to W_\ast$ and $g_\ast:W_\ast \to V_\ast$ so that $g_{r+\epsilon} f_r = \iota^V_{r,r+2\epsilon}$ and $f_{r+\epsilon} g_r = \iota^W_{r,r+2\epsilon}$ for all parameters $r$.

The interleaving distance \cite{chazalProximityPersistenceModules2009} on persistence modules $V_\ast$ and $W_\ast$ is
\begin{equation}
d_{\mathrm{I}}(V_\ast, W_\ast) = \inf \{\epsilon \ge 0 \mid \text{$V_\ast$ and $W_\ast$ are $\epsilon$-interleaved}\}
\end{equation}
This notion of distance satisfies triangle inequality through the composition of interleavings.  Note that the addition or removal an arbitrary number of zero-length pairs to a persistence vector space $V_\ast$ to obtain $V'_\ast$ results in $d_I(V_\ast, V'_\ast) = 0$.

The construction of general interleavings, let alone those that would realize the interleaving distance, can be a daunting task.  Fortunately, for 1-parameter persistent homology the interleaving distance $d_I$ is equivalent to the geometric (and easily computable) bottleneck distance $d_B$ on persistence diagrams \cite{lesnickTheoryInterleavingDistance2015}.

The bottleneck distance considers the birth-death pairs $\{(b_i,d_i)\}_{i\in I}$ as points in the 2-dimensional plane.  The persistence diagram $\dgm(V_\ast)$ is the union of this discrete multi-set of the points $\{(b_i,d_i)\}$ with the diagonal $\Delta = \{(x,x) \mid x\in \mathbb{R}\}$ where points in $\Delta$ are counted with infinite multiplicity.

\begin{definition}
A matching between two persistence diagrams $\dgm_{1}$ and $\dgm_{2}$ is a subset $\Omega \subseteq \operatorname{dgm}_{1} \times \dgm_{2}$ such that every points in $\dgm_{1} \setminus \Delta$ and $\dgm_{2} \setminus \Delta$ appears exactly once in $m$. 
\end{definition}

\begin{definition}\label{def:bottleneck_dist}
The Bottleneck distance between $\dgm_{1}$ and $\dgm_{2}$ is then defined by

\begin{equation}\label{eq:bottleneck_distance}
\mathrm{d}_{\mathrm{B}}\left(\dgm_{1}, \dgm_{2}\right)=\inf _{\text {matching }} \max _{(p, q) \in \Omega}\|p-q\|_{\infty}
\end{equation}
\end{definition}

The bottleneck distance on persistence diagrams is isometric to the interleaving distance on persistence vector spaces, which is a result known as the isometry theorem:
\begin{theorem} \cite{lesnickTheoryInterleavingDistance2015}
Let $\dgm(V_{\ast})$ and $\dgm(W_{\ast})$ be the persistent diagrams of $V_{\ast}$ and $W_{\ast}$ respectively. Then
$$
d_{\mathrm{I}}(V_{\ast}, W_{\ast})=d_{\mathrm{B}}(\dgm(V_{\ast}), \dgm(W_{\ast}))
$$
\end{theorem}

The matching in the bottleneck distance actually gives an interleaving which maps a vector associated to a persistence pair  in $\dgm(V_\ast)$ to the vector associated with the matched pair in $\dgm(W_\ast)$ which realizes the interleaving distance.

% Brad: not sure if we need this definition:
% \begin{definition}
% Given $p \geq 1$, the Wasserstein distance is defined by
% $$
% W_{p}\left(\operatorname{dgm}_{1}, \operatorname{dgm}_{2}\right)^{p}=\inf _{\text {matching } m} \sum_{(p, q) \in m}\|p-q\|_{\infty}^{p} .
% $$
% \end{definition}

\subsection{Bounds on the Interleaving Distance}
% Brad: key points: bound from canonical simplicial map, bound from GH distance, bound from Hausdorff distance.
In practice, persistent homology of the Vietoris-Rips filtration can be quickly approximated using sub-sampling.  Bounds on this approximation come from the Hausdorff or Gromov-Hausdorff distance on between point clouds \cite{chazalGromovHausdorffStableSignatures2009}. 

\begin{definition}
The Hausdorff distance between two subsets $X$ and $Y$ within the same metric space is
$$
d_{H}(X, Y)=\max \left\{\sup _{x} \inf _{y}\|x-y\|_{\infty}, \sup _{y} \inf _{x}\|y-x\|_{\infty}\right\} 
$$
\end{definition}

\begin{definition}
A correspondence between two sets $X$ and $Y$ is a subset $C \subset X \times Y$ such that: $\forall x \in X, \exists y \in Y$ s.t. $(x, y) \in C$, and $\forall y \in Y, \exists x \in X$ s.t. $(x, y) \in C$. The set of all correspondences between $X$ and $Y$ is denoted by $\mathcal{C}(X, Y)$.
\end{definition}

\begin{definition}
The Gromov-Hausdorff distance between compact metric spaces $\left(X, \mathrm{~d}_{X}\right),\left(Y, \mathrm{~d}_{Y}\right)$ is:
$$
\mathrm{d}_{\mathrm{GH}}\left(\left(X, \mathrm{~d}_{X}\right),\left(Y, \mathrm{~d}_{Y}\right)\right)=\frac{1}{2} \inf _{C \in \mathcal{C}(X, Y)}\left\|\Gamma_{X, Y}\right\|_{l ^\infty(C \times C)},
$$
where $\Gamma_{X, Y}: X \times Y \times X \times Y \rightarrow \mathbb{R}^{+}$ is defined by $\left(x, y, x^{\prime}, y^{\prime}\right) \mapsto\left|\mathrm{d}_{X}\left(x, x^{\prime}\right)-\mathrm{d}_{Y}\left(y, y^{\prime}\right)\right|$ and the notation $\left\|\Gamma_{X, Y}\right\|_{l^{\infty}(C \times C)}$ stands for $\sup _{(x, y),\left(x^{\prime}, y^{\prime}\right) \in C} \Gamma_{X, Y}\left(x, y, x^{\prime}, y^{\prime}\right)$.
\end{definition}  

\begin{theorem}\label{theorem:bottleneck less than GH}
\cite{chazalGromovHausdorffStableSignatures2009} For any finite metric spaces $\left(X, \mathrm{~d}_{X}\right)$ and $\left(Y, \mathrm{~d}_{Y}\right)$, for any $k \in \mathbb{N}$, the bottleneck distance between two $k$-th persistent diagrams of Rips filtrations is bounded by the Gromov-Hausdorff between two spaces
$$
\begin{aligned}
\dB (\dgm(X), \dgm(Y))
\leq \mathrm{d}_{\mathrm{GH}}\left(\left(X, \mathrm{~d}_{X}\right),\left(Y, \mathrm{~d}_{Y}\right)\right) .
\end{aligned}
$$
\end{theorem}

The theorem above can also provide us an interleaving/bottleneck distance bound for sub-sampled data sets only if we can find their Gromov-Hausdorff distance. 

\begin{corollary}\label{corollary:bottleneck bound}
Let $X$ be a dataset and $Y$ be a low-dimensional embedding of $X$, and $X_{\text{sub}}$ and $Y_{\text{sub}}$ are their sub-sampled data sets. Then
$$
\begin{aligned}
d_{\mathrm{B}}(\dgm(X), \dgm(Y)) 
& \leq d_{\mathrm{GH}}(X, X_{\text{sub}}) \\
& + d_{\mathrm{B}}(\dgm(X_{\text{sub}}), \dgm(Y_{\text{sub}})) \\
& + d_{\mathrm{GH}}(Y, Y_{\text{sub}}) 
\end{aligned}
$$
\end{corollary}
The proof only needs the triangle inequality of metrics. 

% However, this Corollary provides an important bound for the performance of our dimension reduction algorithm, which is novel compared to the existing discussion using Wasserstein distance (cite). Another reason for using Bottleneck over Wasserstein distance is that Bottleneck distance focus on the most prominent feature in persistent diagrams, while Wasserstein would pay attention to some noise features close to the diagonal line that should be ignored.

% \subsection{Bounds and Approximations}

% Explain computational considerations
% GH-bound, Hausdorff bound
% Sparse Rips

% TODO: how does knowing the interleaving distance help us interpret topological structures?
\section{Measuring Embedding Distortion through Interleavings}\label{section: Measuring Embedding Distortion through Interleavings}
%Explain measuring distortion through interleaving distance, interpretability of robust features, and have an example where we compute interleaving distance on data set using a variety of methods (e.g. in sklearn dimension reduction + UMAP).  Include our method.

The interleaving distance provides a natural way to measure how accurately a transformation of a data set $X$ with dissimilarity $d_X$ into a low-dimensional embedding $Y$ with dissimilarity $d_Y$ distorts topology.  In particular, it provides a way to eliminate topological type-I errors and reduce topological type-II errors when inferring information about the space $(X, d_X)$ via the embedding $(Y,d_Y)$, as might be done in data visualization.

In order to have a notion of topological error, we must select topological features of $(Y,d_Y)$ which are believed to be significant, meaning that they are believed to correspond to topological structures in $(X,d_X)$. 
% \textcolor{red}{Repeated sentence as above: We will consider selection homological features, meaning that we would like to select persistent homology classes of $\calR(Y,d_Y)$ which are in correspondence with persistent homology classes of $\calR(X, d_X)$.}
\begin{definition}
A topological type-I error in $(Y,d_Y)$ is the selection of a feature in $(Y,d_Y)$ which has no corresponding feature in $(X,d_X)$.
\end{definition}
For example, if the embedding $(Y,d_Y)$ splits a single cluster in $(X,d_X)$ into two clusters, then making a distinction between the two clusters by selecting both $H_0$ pairs in $(Y,d_Y)$ would be a topological type-I error.
\begin{definition}
A topological type-II error in $(Y,d_Y)$ is made when a structure which corresponds to a structure in $(X,d_X)$ is not selected as significant.
\end{definition}
For example, if two clusters in $(X,d_X)$ merge into a single cluster in $(Y,d_Y)$ then we are forced to make a topological type-II error since we can select at most one persistent $H_0$ feature in $(Y,d_Y)$.

% \textcolor{red}{TODO: relate this feature to topological bootstrapping/hypothesis testing, although there is no randomness here.}

\subsection{Selection of Homological Features}
Let $H_k(X;r) = H_k(\calR(X,d_X;r))$ denote the $k$-dimensional homology of the Vietoris-Rips complex at parameter $r$, and $H_k(X)$ denote the $k$-dimensional persistent homology of the Vietoris-Rips filtration.  Similarly, we have $H_k(Y;r)$ and $H_k(Y)$.  We refer to each persistence pair in $H_k(X)$ or $H_k(Y)$ as a homological feature of $(X,d_X)$ or $(Y,d_Y)$ respectively.  We would like to select features of $(Y,d_Y)$ which are in correspondence with features of $(X,d_X)$.  
A simple selection procedure is to compute the interleaving distance $\epsilon = d_{\mathrm{I}}(H_k(\calR(Y, d_Y), H_k(\calR(X,d_X))$, and to select any homology class $(b,d)\in H_k(\calR(Y, d_Y))$ with $|d - b| > 2\epsilon$.  

\begin{proposition}\label{prop:no_typeI_errors}
No type-I errors are made using this selection procedure.
\end{proposition}
\begin{proof}
Suppose that a selected homology vector has birth $b$ and death $d$ in $H_k(Y)$ with $|d - b| > 2\epsilon$.
We consider the $\epsilon$-shift maps that realize the interleaving
\begin{equation}
\begin{tikzcd}[column sep=-2em]
	{H_k(Y; b)}\ar[rr]\ar[dr] & {} & {H_k(Y; b+2\epsilon)} \\
	& {H_k(X; b + \epsilon)}\ar[ur]
\end{tikzcd}
\end{equation}
Because $|d-b| > 2\epsilon$ and the above diagram commutes, the selected vector must have a non-zero image in $H_k(X;b+\epsilon)$.

Furthermore, if two selected vectors have the same image in $H_k(X)$ then their difference must be zero in the image back in $H_k(Y)$.
\end{proof}

\begin{proposition}\label{prop:type2_control}
Every persistence pair in $H_k(X)$ with $|d-b|> 4\epsilon$ has a corresponding persistence pair in $H_k(Y)$ which is selected.
\end{proposition}

\begin{proof}
Suppose that a homology vector in $H_k(X)$ has birth $b$ and death $d$, with $|d-b|> 4\epsilon$.  Then in the following diagram the associated vector must have non-zero image in each vector space in the $\epsilon$-interleaving
\begin{equation}
\begin{tikzcd}[column sep=-0.7em]
	{H_k(X;b)}\ar[rr]\ar[dr] && {\cdots}\ar[rr]\ar[dr] && {H_k(X; d)} \\
	& {H_k(Y; b + \epsilon)}\ar[rr]\ar[ur] &&{H_k(Y; d - \epsilon)}\ar[ur]
\end{tikzcd}
\end{equation}
Which implies that the persistence pair in $H_k(X)$ is in correspondence with a persistence pair $(b',d')\in H_k(Y)$, where $b'\le b+\epsilon$ and $d'\ge d-\epsilon$.  This pair has length $>(d-\epsilon) - (b+\epsilon) = 2\epsilon$, so is selected by our procedure.
\end{proof}

As a result, not only does this selection procedure guarantee that we will make no topological type-I errors, but we also will not make any type-II errors involving a homological feature of $(X,d_X)$ of sufficient size.  Note that if we lower our threshold for selecting pairs of $H_k(Y)$ 
% Yuan: , then we could introduce ...
then we could introduce the possibility of homological type-I errors, and that there is always the possibility of homological type-II errors when neither pair in the correspondence has sufficient length.

\subsection{Application to General Dimension Reduction}

\begin{figure*}[ht!]
        \centering
            \subfigure[Dimension reduction result] 
            {
                \label{subfig:tendril_pca}
                \includegraphics[width=.2\textwidth]{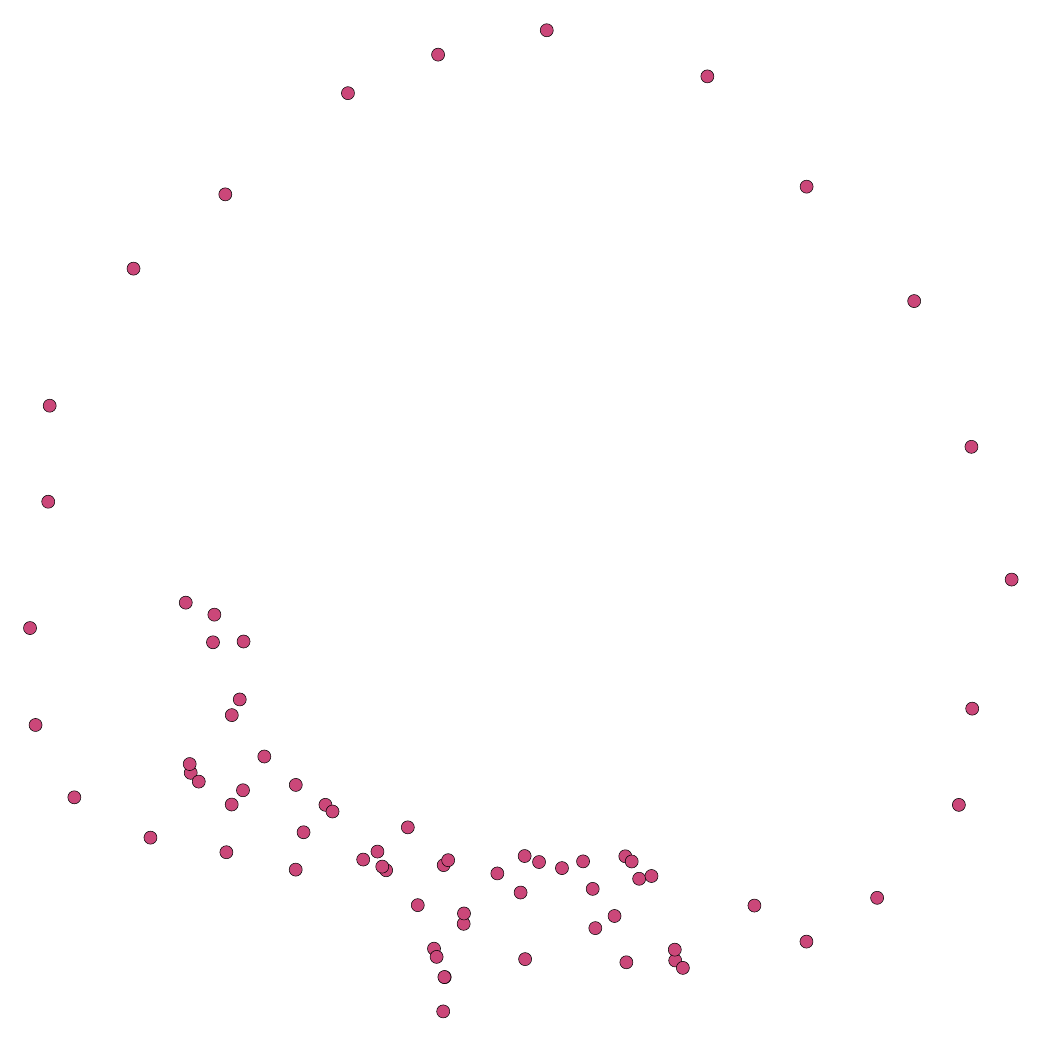} % .png .jpg ... according to supported graphics files
            }
            \subfigure[Persistence diagram with unconfident band] 
            {
                \label{subfig:tendril_isomap}
                \includegraphics[width=.2\textwidth]{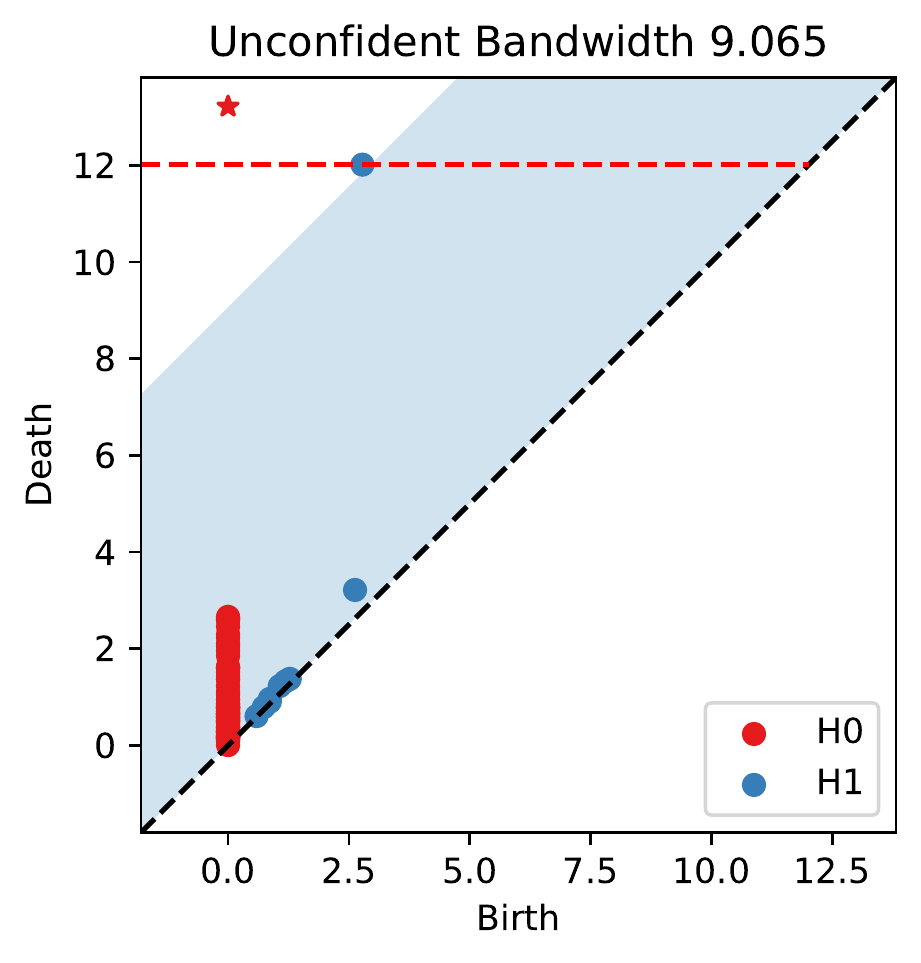} % .png .jpg ... according to supported graphics files
            }
            \subfigure[Longest uncertain H1 representative (Blue) and certain H1 representative (Red)] 
            {
                \label{subfig:tendril_pca}
                \includegraphics[width=.2\textwidth]{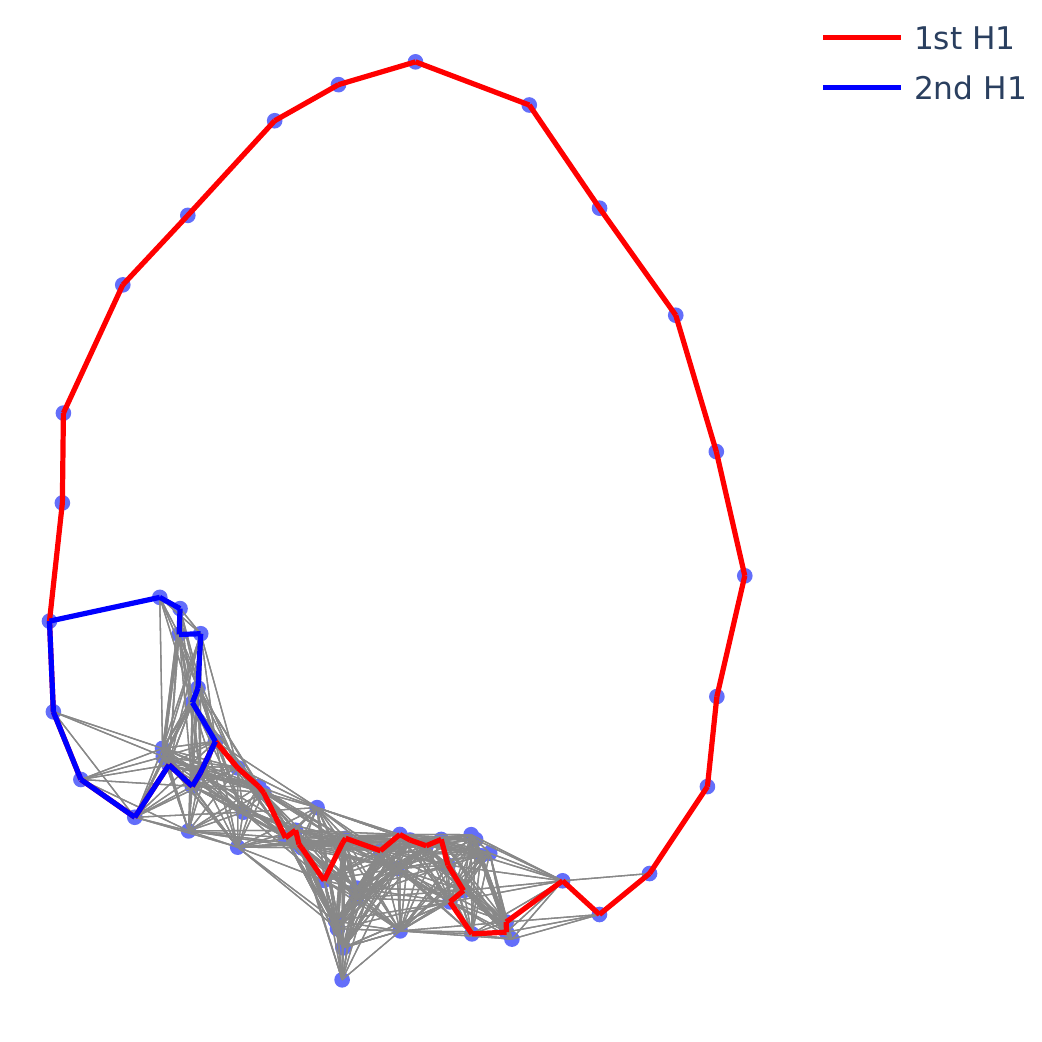} % .png .jpg ... according to supported graphics files
            }
        \caption{Dimension reduction results of Tomato dataset by first using PCA to reduce dimension to 10 and then using our PH optimization to reduce to 2.}
        \label{fig:tomato_viz}
\end{figure*}
This selection procedure can be used to assess how well any transformation of point cloud data $X$ preserves topology.  In particular, we can compute the bottleneck distance $\epsilon = d_\mathrm{I}(H_k(X),H_k(Y))$, the number of features of $H_k(Y)$ with $|d-b| > 2\epsilon$ and the number of features of $H_k(X)$ with $|d-b| > 4\epsilon$.  In the context of dimension reduction, it would be desirable to minimize the interleaving distance in order to maximize the number of features we can identify which are in correspondence with features in the original data set.  This is the approach we pursue in \cref{sec:interleaving_opt}.

In table \cref{tab:tomato_selected} we compare several algorithms for dimension reduction on a set of images from the Columbia Object Image Library (COIL-100) \cite{Nene96objectimage} taken of a tomato at various angles in a circle which has a single large $H_1$ homological feature in the original data. Methods compared include PCA, MDS \cite{kruskalMultidimensionalScalingOptimizing1964}, and ISOMAP \cite{tenenbaumGlobalGeometricFramework2000} with a method based on minimizing the bottleneck distance developed in \cref{sec:interleaving_opt}, PH, and a hybrid PH + PCA.  Because every Vietoris-Rips filtration has a single $H_0$ pair with death at $\infty$, at least one $H_0$ feature will always be selected.  Only PH + PCA allows for the selection of an $H_1$ feature.  Visualization of each embedding can be found in the appendix, and a more detailed visualization of the PH + PCA embedding can be found in \cref{fig:tomato_viz}.

\begin{table}[h]
    \centering
    \begin{tabular}{c|c|c|c|c|c}
    Method & max $H_1$ & $\dI$ $H_0$ & $\dI$ $H_1$ & $H_0$ & $H_1$  \\
    \hline
    PCA & 4.854 & 5.148 & 10.852 & 1 & 0\\
    MDS & 5.626 & 4.782 & 10.217 & 1 &0\\
    ISOMAP & 113.968 & 1.935 & 108.623 & 1 &0\\
    PH & 7.543& 5.295 &5.295 & 1 & 0\\
    PH + PCA & 9.234 &4.689 &4.532 & 1 & 1\\
\end{tabular}
    \caption{Selection of topological features.  $\max H_1$ is the length of the largest $H_1$ pair in $Y$. The last two columns indicate the number of features which are in correspondence with the original data in $H_0$ and $H_1$ via the interleaving.}
    \label{tab:tomato_selected}
\end{table}

\begin{figure}[ht]
\centering 
    \includegraphics[width=0.3\linewidth]{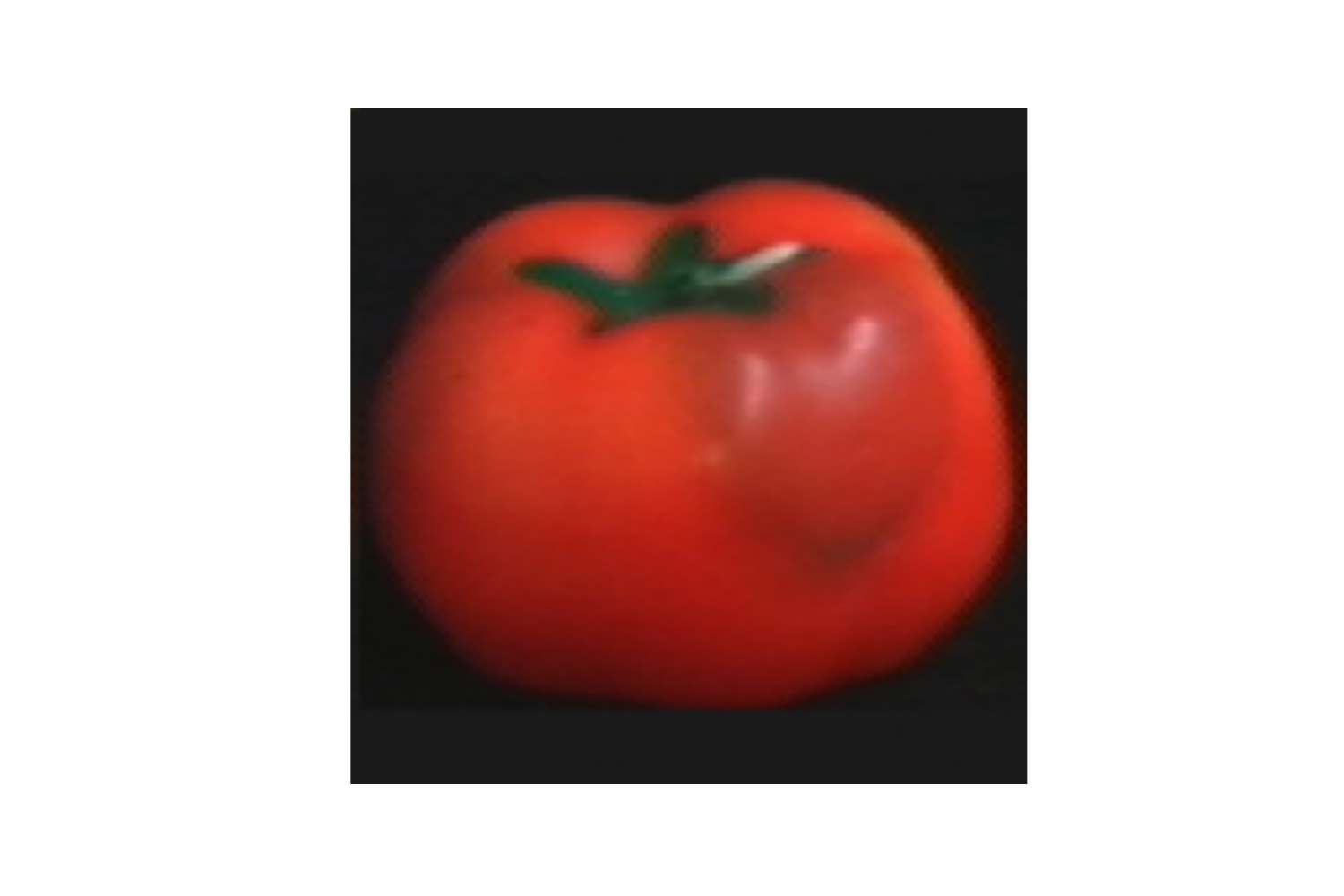}
    \includegraphics[width=0.3\linewidth]{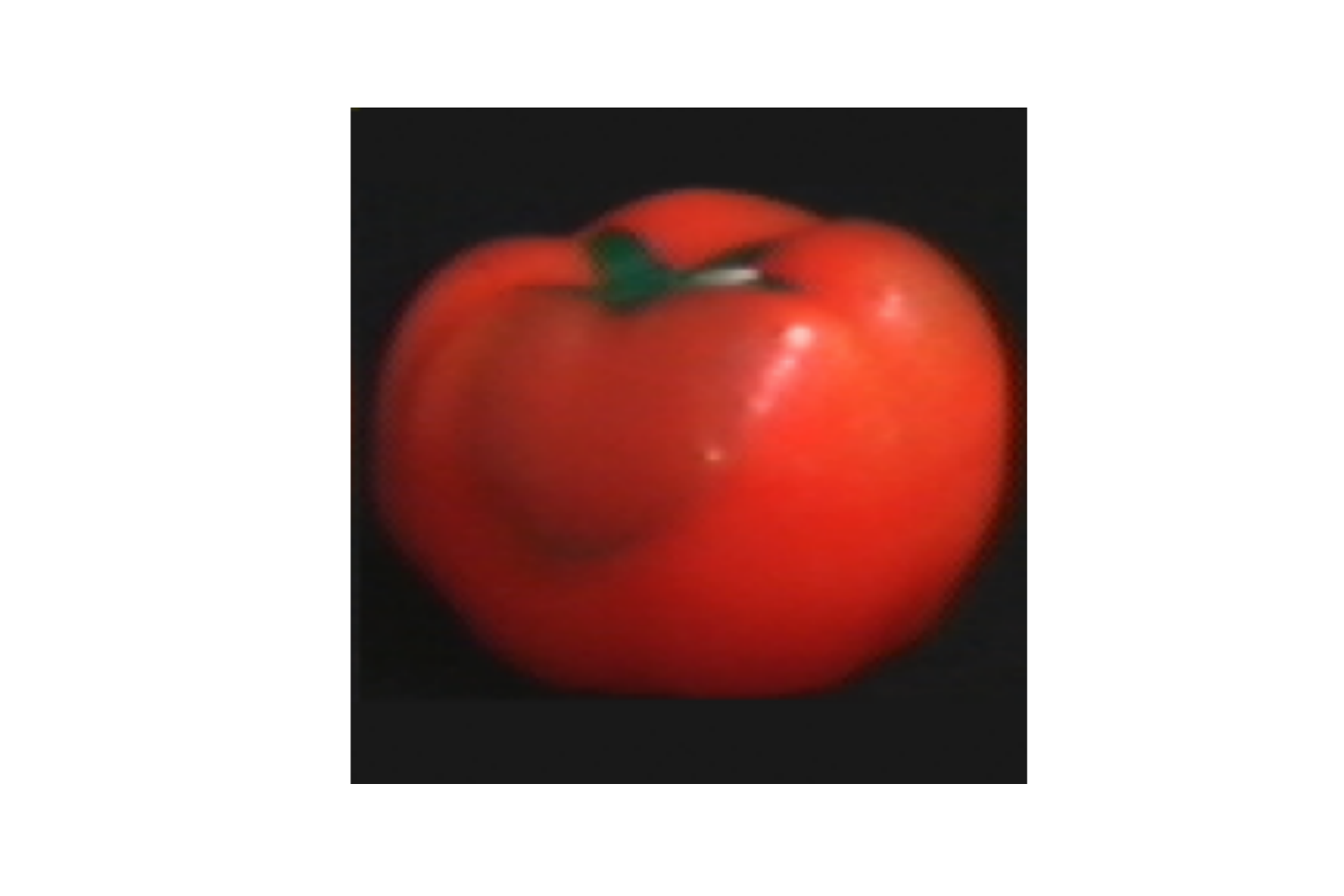}
    \includegraphics[width=0.3\linewidth]{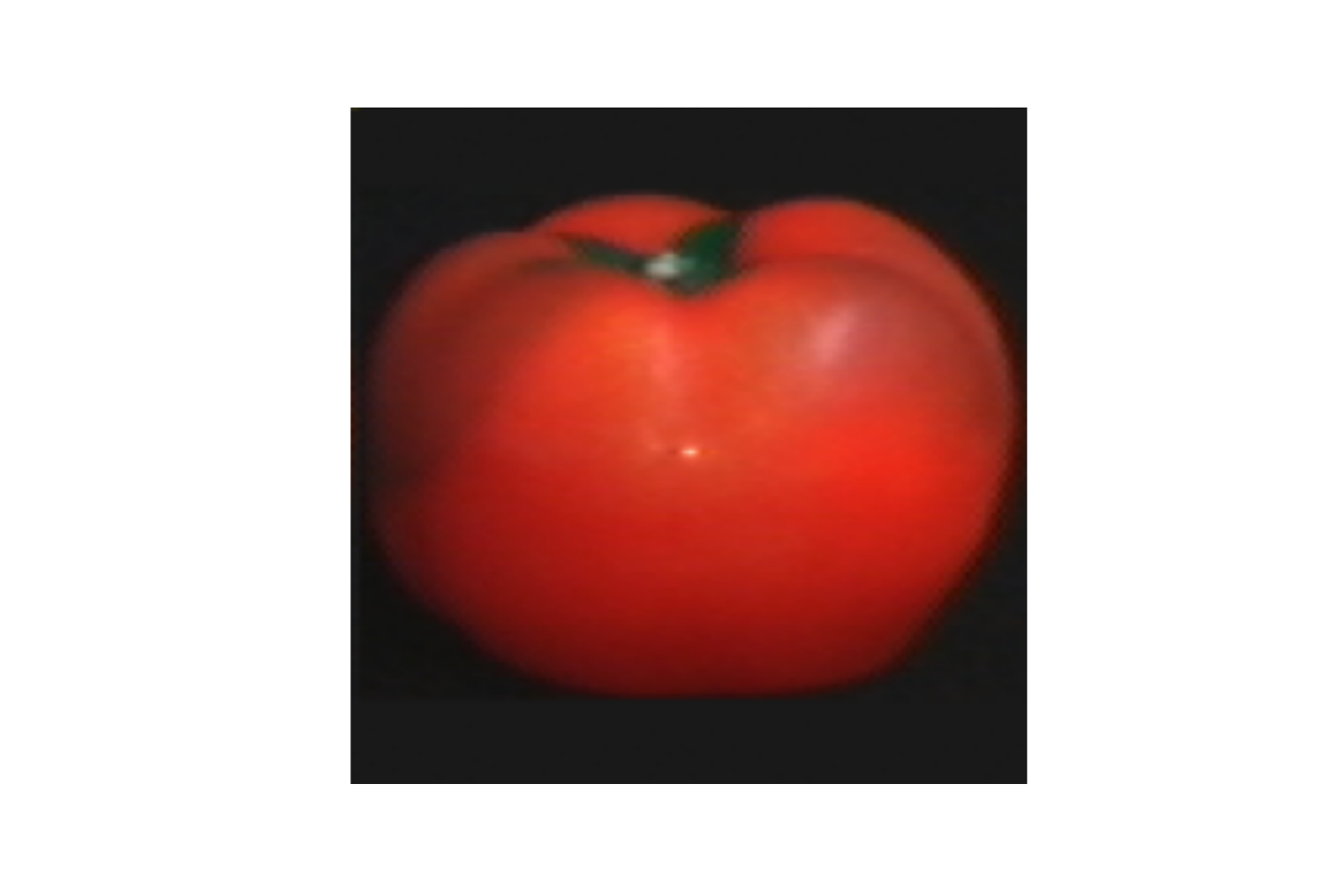}
    \caption{Tomato pictures token from 0, 50, and 100 degree angles. In total, there are 72( = $\frac{360}{5}$) different angles.}
    \vspace{-1.5ex}
    \label{fig:tomato}
\end{figure}

\subsection{Homological Caveats}
Some care must be taken in interpreting the interleaving as presented here.  Importantly, the correspondence between persistence pairs of $H_k(X)$ and $H_k(Y)$ is entirely algebraic; there are not necessarily topological 
shift maps $f:\calR(X; r) \to \calR(Y;r+ \epsilon)$ 
and $g: \calR(Y; r) \to \calR(X;r+ \epsilon)$ 
which induce the interleaving on homology.
Even if the interleaving distance between $X$ and $Y$ is zero, there is no guarantee that there is a natural topological map between the two spaces.  There are many possible spaces which have identical persistent homology \cite{curryFiberPersistenceMap2018}, and to maintain some level of geometric interpretability of persistence pairs of $Y$ in terms of the persistence pairs of $X$, it is desirable to incorporate additional constraints onto the embedding $Y$, as is often the case in dimension reduction algorithms.

\section{Optimizing Interleaving Distance}\label{sec:interleaving_opt}

We now turn to explicitly optimizing an embedding $Y$ to minimize the interleaving distance to the original data set $X$.

\subsection{Optimizing Persistent Homology}\label{sec:opt_ph}
Gradient-based optimization techniques can be applied to persistent homology by backpropagating the gradient of a function of the persistence pairs back to the input values of a filtration.  This is often done by considering a featurization of the persistence pairs such as algebraic functions of the pairs \cite{gabrielssonTopologyLayerMachine2020} or persistence landscapes \cite{kimPLLayEfficientTopological2020,carrierePersLayNeuralNetwork2020a}, but in our situation, we will use the bottleneck distance $\dB(\dgm(X),\dgm(Y))$.

Optimization with Vietoris-Rips filtrations is described in detail in \cite{gabrielssonTopologyLayerMachine2020}, which we summarize here.  The key is that every simplex addition in the filtration either creates or destroys homology and the corresponding birth or death takes that filtration value.  If $f$ is a function of persistence pairs, this allows for the mapping of $\partial f/\partial b$ or $\partial f/\partial d$ to $\partial f/\partial w_\sigma$ where $w_\sigma$ is the filtration value of the unique simplex $\sigma$.  In the case of Vietoris-Rips filtrations, the filtration value of a simplex $(x_0,\dots,x_k)$ is the maximum pairwise distance $d_Y(x_i,x_j)$ where $x_i,x_j$ are vertices in the simplex, so this can then be backpropagated to a gradient $\partial f/\partial d_Y(x_i,x_j)$.  There is a potential issue here which is that a single edge may map to multiple higher-order simplices, but \cref{thm:descent_dir} indicates that this is not generally a problem.   Finally, if we choose a differentiable metric on $Y$ such as the Euclidean metric, we can backpropagate the gradient to point locations in the embedding. 

% In our notation, 
% \begin{itemize}
%     \item $ids$ stores the indices of subsampled data points.
%     \item $Hds[k]$ stores the Hausdorff distance from subset $X[ids[:k+1]]$ to $X$. 
%     \item $ds$ stores the distance of each point to the subsampled set of . 
% \end{itemize}

% \subsection{Inverse map}\label{subsection:Inverse map}

% Given a filtration\footnote{We assume all simplices are ordered in the filtration, even if some of them appear at the same filtration value.}, we will first introduce a map that will map each persistent birth-death pair $(b,d)$\footnote{The death time $d$ can also take $\infty$.} to a pair of simplices $(\tau, \sigma)$. The correspondence is unique and iterates all simplices. 

% The reason why this map exist comes from the simple observation of computation of persistent homology. From the viewpoint of matrix\cite{Carlsson2019PersistentAZ}, since the rank of $k$-homology just a subtraction of the nullity of boundary matrix from the rank of boundary matrix with one dimension higher, adding a simplex is equivalent to add a column or a row depending on the dimension. Thus, every simplex either creates or destroys a homology class. Then, the map $(b,d) \rightarrow (\tau, \sigma)$ is defined by the persistent homology computed on the given filtration and we call it inverse map as introduced in \cite{gabrielssonTopologyLayerMachine2020}. 

\subsection{Optimizing the Bottleneck Distance}

We are interested in optimizing the embedding $Y$ to minimize the interleaving distance via the bottleneck distance $\dB(\dgm(Y), \dgm(X))$.  Because the original data set $X$ is fixed, we have a function $f(\dgm(Y)) = \dB(\dgm(Y), \dgm(X))$ which we can fit into the optimization framework for persistent homology.

First, we recall that the bottleneck distance uses a matching between persistence pairs of $Y$ and $X$ and the diagonal $\Delta$ representing potential zero-length pairs, and that the distance is computed from the maximum-weight matching. This leads to three possibilities. 
\begin{enumerate}
% Yuan: might need to clearly state what is the maximum and minimum weight here.
\setlength\itemsep{0em}
    \item The maximum weight matching occurs between two non-diagonal pairs.
    \item The maximum weight matching occurs between a pair in $\dgm(Y)$ and the diagonal in $\dgm(X)$.
    \item The maximum weight matching occurs between the diagonal in $\dgm(Y)$ and and a pair in $\dgm(X)$.
\end{enumerate}
The bottleneck distance does not consider matchings between two diagonal points.  In the first two cases, it is possible to find a descent direction.  In the third case, $Y$ is at a local saddle point of the bottleneck distance.

% \begin{proposition}
% % Yuan: which case of the three above are you referring to here?
% If a pair $(d,b)$ in the maximum-weight matching of $\dgm(Y)$ and $\dgm(X)$ has $d=b$, then the backpropagated gradient of $\dB(\dgm(Y), \dgm(X))$ is zero.
% \end{proposition}
% \begin{proof}
% Let $(b,d)\in \dgm(Y)$ denote the matched pair.  If $d=b$, then the pair has zero length, and so matches with the diagonal in $\dgm(X)$. This implies that $\dB(\dgm(Y),\dgm(X)) = 0$, so the bottleneck distance is at a minimum.
% \end{proof}
We will show that under the condition that a single pair attains the maximum-weight matching, we may obtain a descent direction for $\dB$ with respect to the point locations in $Y$.  This is a mild condition observed in practice in our empirical experiments.
\begin{proposition}\label{prop:two_pairwise}
If a single pair $(b,d)$ in the maximum-weight matching of $\dgm(Y)$ and $\dgm(X)$ realizes the bottleneck distance, then the backpropagated $\frac{\partial \dB}{\partial d(y_i,y_j)}$ is non-zero for at most two pairwise distances in $Y$.
\end{proposition}
\begin{proof}
Because a single pair realizes the distance, it must have non-zero length.  In this case, the associated simplex filtration values must map to distinct distances in $Y$.  Let $q$ denote the matched point in $\dgm(X)$, so $\dB(\dgm(Y), \dgm(X)) = \|(b,d) - q\|_\infty$, one or both distances may admit a non-zero gradient.
\end{proof}
\begin{theorem}\label{thm:descent_dir}
If a single pair in $\dgm(Y)$ realizes the bottleneck distance, then $\dB(\dgm(Y),\dgm(X))$ admits a descent direction on the point embedding $Y$.
\end{theorem}
\begin{proof}
From \cref{prop:two_pairwise}, there are at most two pairwise distances which have a non-zero backpropagated gradient.  Let these distances be $d_Y(y_0,y_1)$ and $d_Y(y_2,y_3)$.  If the points $y_0,y_1,y_2,y_3$ are distinct, then we can simply backpropagate the gradient to each point location.  

One point might possibly participate in both distances -- if two points are redundant, then the distances are not distinct.  In this case, let us have $b = d_Y(y_0,y_1)$ and $d=d_Y(y_1,y_2)$.  In this case, we can take a step in the directions $\frac{\partial \dB}{\partial y_0} = \frac{\partial \dB}{\partial b} \frac{\partial b}{\partial y_0}$, $\frac{\partial \dB}{\partial y_2} = \frac{\partial \dB}{\partial d} \frac{\partial d}{\partial y_2}$, and leave $y_1$ unchanged.  Because $\frac{\partial d_Y(y_0,y_1)}{\partial y_0}$ and $\frac{\partial d_Y(y_0,y_1)}{\partial y_1}$ are non-zero, and at least one of $\frac{\partial \dB}{\partial d}$ or $\frac{\partial \dB}{\partial d}$ is non-zero, this step direction will decrease the bottleneck distance.
\end{proof}

\begin{proposition}\label{prop:db_subgradient}
$\nabla_Y \dB(\dgm(Y), \dgm(X))$ admits a generalized subdifferential.
\end{proposition}
\begin{proof}
Any persistence pair in $\dgm(Y)$ that does not participate in the bottleneck matching can be perturbed without affecting the bottleneck distance.
\end{proof}

\cref{prop:db_subgradient} implies that we have great freedom to perturb the embedding $Y$ while minimizing the bottleneck distance.  In the third case above, when the diagonal of $\dgm(Y)$ is involved in the maximum weight matching, we have the freedom to perturb $Y$ in any direction.  This allows for easy optimization of the bottleneck distance in conjunction with a secondary objective.

% In order to optimize the bottleneck distance as an objective, we must compute the gradient with respect to births and deaths in persistent homology.  The 
% The gradient of the bottleneck distance \cref{eq:bottleneck_distance} with respect to births and deaths depends on the matching of points.  There are three cases:

% We will assume diagrams are non-empty, so we won't consider a matching between two diagonal points.

% Case for perfect matching: do not take diagonal as rep.

% In final case, all directions are in subgradient.

% Point is that we can find an element of the subgradient which optimizes a secondary objective.

% \subsection{Sub-sampling}\label{subsection:Sub-sampling}
% The number of points often makes the bottlenecks in persistent homology computation. In our dimension reduction method, since optimization requires repeated computation of persistent homology, controlling the number of points is crucial. We utilize the greedy sub-sampling algorithm (cite), which will greedily find the furthest point at each iteration and in the end return the Hausdorff distance between the sub-sampled and the whole data set. However, recall that \Cref{corollary:bottleneck bound} requires Gromov–Hausdorff distance which requires solving quadratic assignment problem. We then bound the Gromov–Hausdorff distance by Hausdorff returned from the greedy sampling algorithm. 

\section{Experiments}

% \begin{figure*}[ht!]
%         \centering
%             \subfigure[PCA] 
%             {
%                 \label{subfig:tendril_pca}
%                 \includegraphics[width=.15\textwidth]{images/tendril/PCA.pdf} % .png .jpg ... according to supported graphics files
%             }
%             %
%             \subfigure[ISOMAP] 
%             {
%                 \label{subfig:tendril_isomap}
%                 \includegraphics[width=.15\textwidth]{images/tendril/ISOMAP.pdf} % .png .jpg ... according to supported graphics files
%             }
%             %
%             \subfigure[PH] 
%             {
%                 \label{subfig:tendril_PH}
%                 \includegraphics[width=.15\textwidth]{images/tendril/PH.pdf}
%             }
%             %
%         \caption{Dimension Reduction Results of 5-Tendril by PCA, MDS and PH}
%         \label{fig:Tendril Result}
% \end{figure*}
% \begin{figure*}[ht!]
\begin{figure}
        \centering
                \includegraphics[width=.15\textwidth]{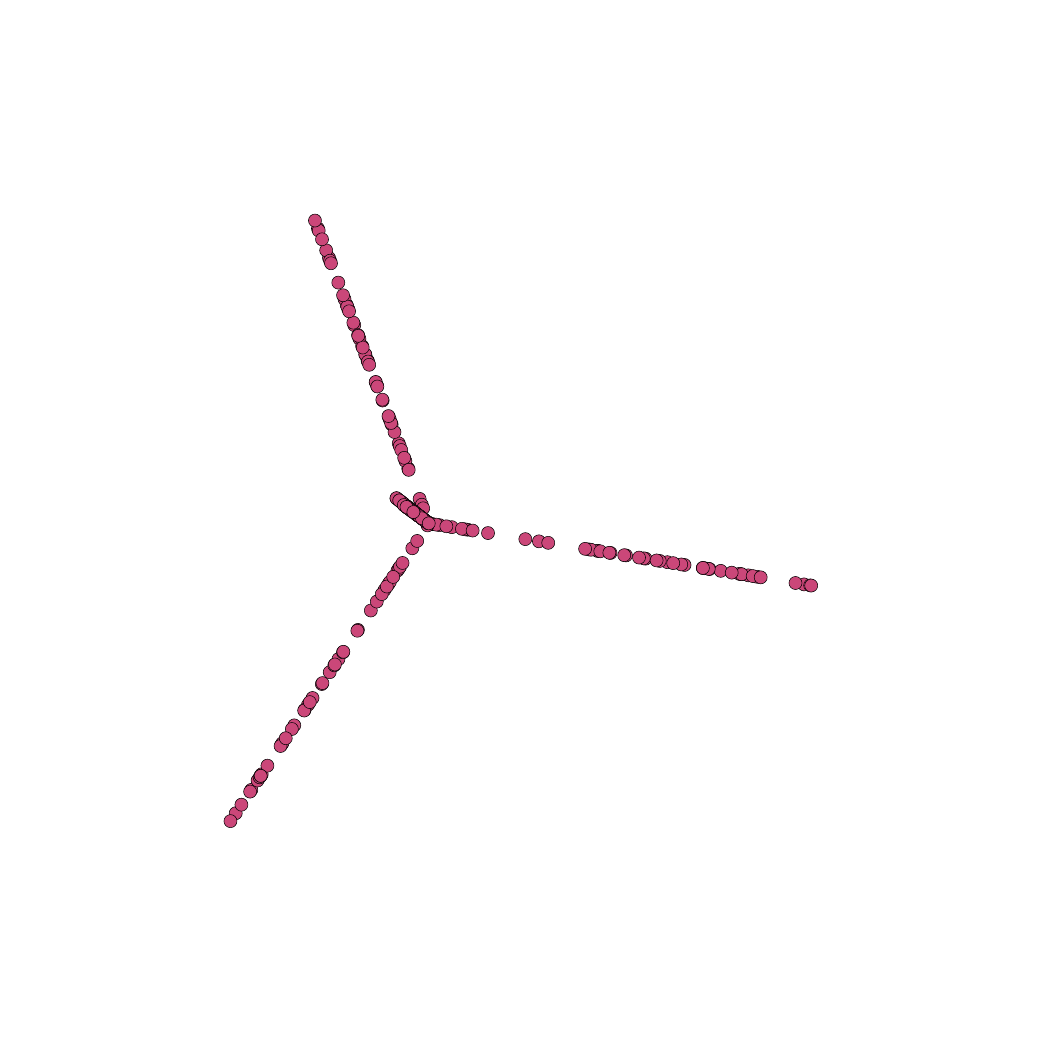} % .png .jpg ... according to supported graphics files
                \includegraphics[width=.15\textwidth]{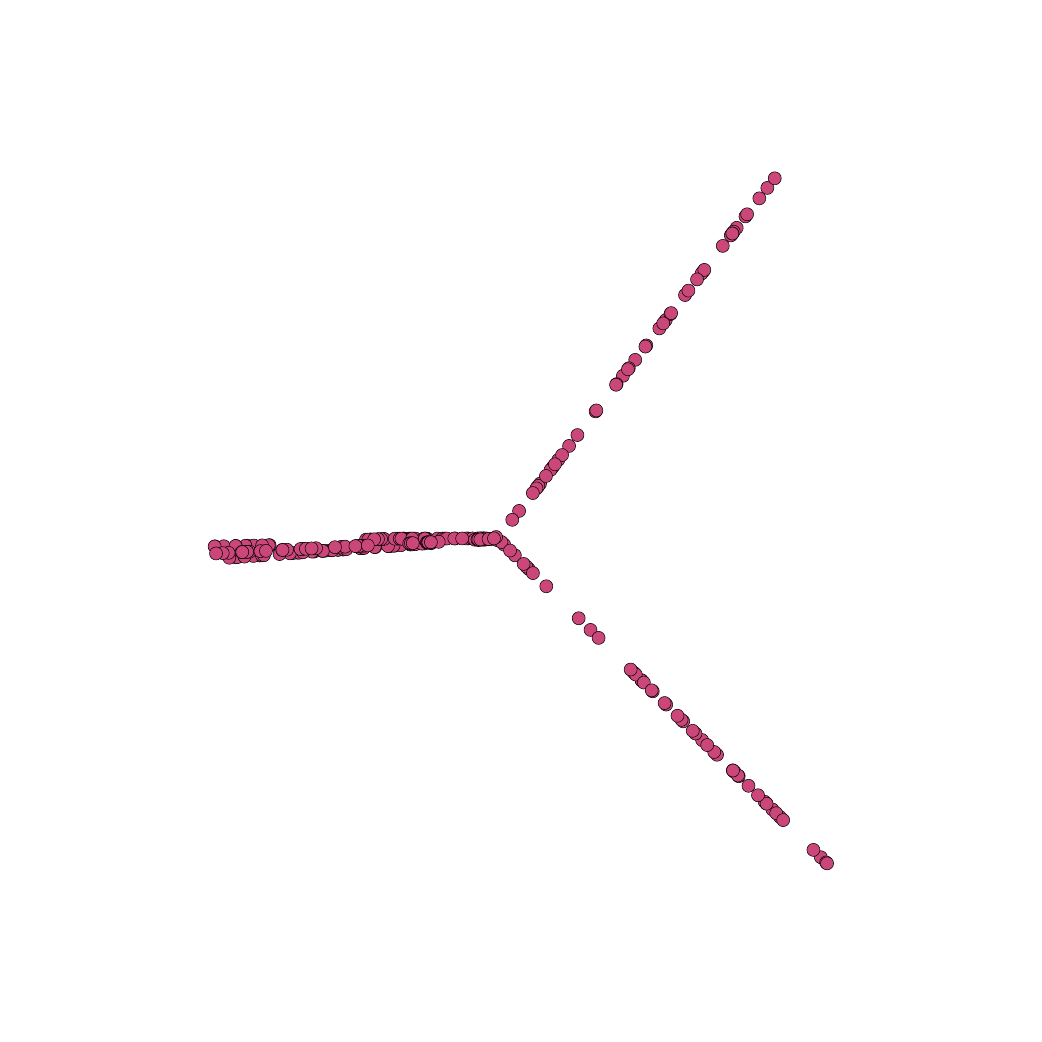} % .png .jpg ... according to supported graphics files
                \includegraphics[width=.15\textwidth]{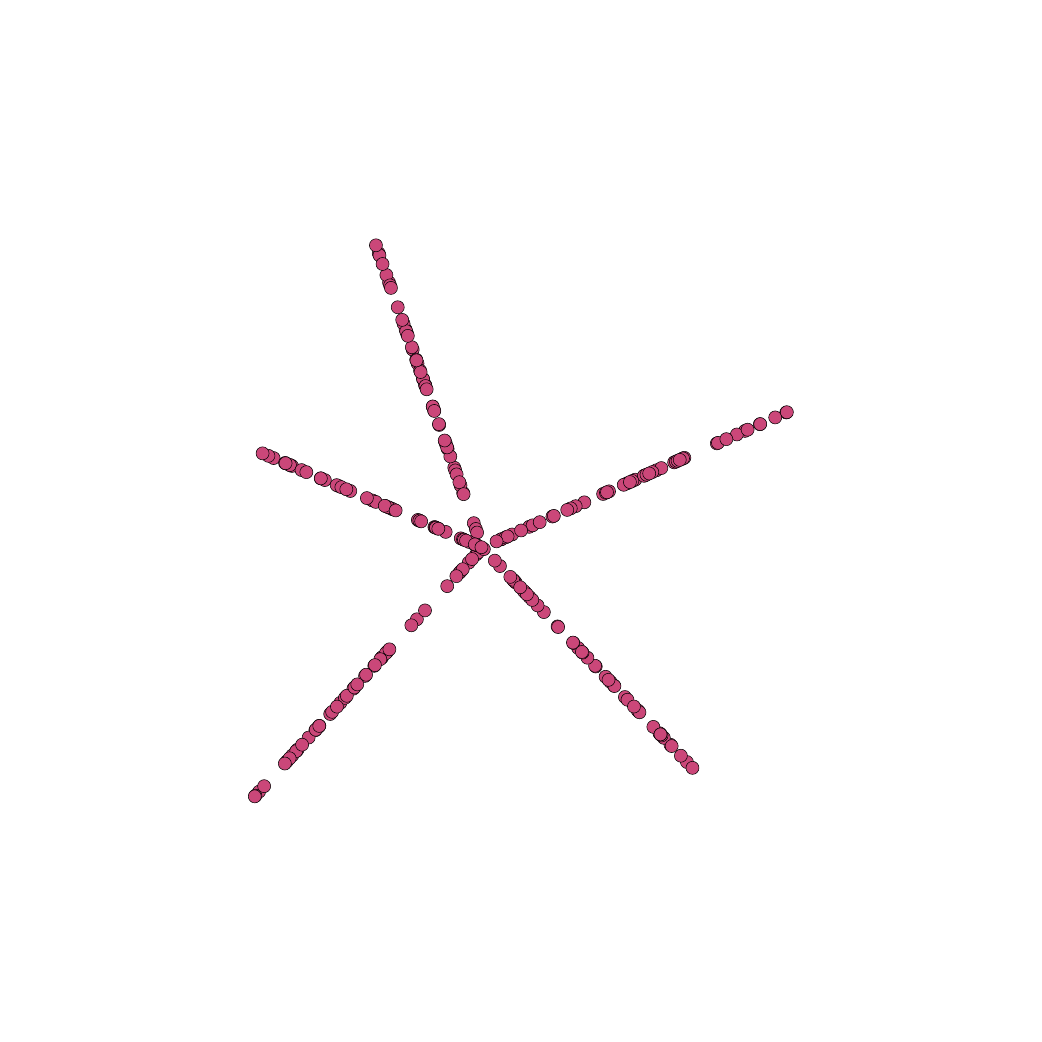}
        \caption{Dimension reduction results of 5-Tendril left to right: PCA, MDS and PH}
        \label{fig:Tendril Result}
\end{figure}

In this section, we will introduce the experiments and results using our dimension reduction method. For convenience, we will call our method PH optimization as it based on persistent homology. 

We implement a form of projection pursuit \cite{friedmanProjectionPursuitAlgorithm1974} which seeks to find a linear projection $P$ which minimizes the bottleneck distance between $X$ and $Y=XP$. We add the orthogonality constraint to the optimization candidate space (also called Stiefel manifold) and the optimization gradient descent algorithm will use the Cayley transform introduced in \cite{Wen2013AFMStifelOpt}.
% \footnote{Note that there are also recent progress of optimization on Stiefel manifold like \cite{Li2020EfficientRO} to speedup by approximation, but currently speed is not the main point of this paper}

In our dimension reduction method, since optimization requires repeated computation of persistent homology, controlling the number of points is crucial.
% We utilize the greedy sub-sampling algorithm \textbf{cite} , which will greedily find the furthest point at each iteration and in the end return the Hausdorff distance between the sub-sampled and the whole data set.
In an effort to obtain a sample with close persistent homology, we use a greedy strategy based on minimizing the Hausdorff distance from the sample to the full point cloud.
However, recall that \Cref{corollary:bottleneck bound} requires Gromov–Hausdorff distance which requires solving quadratic assignment problem. We will bound the Gromov–Hausdorff distance by Hausdorff returned from the greedy sampling algorithm. 

Our procedure is implemented in Pytorch, which supports automatic differentiation without explicitly passing gradients once we have defined two layers: a Vietoris-Rips layer and a bottleneck distance layer. 
The Rips layer based on BATS
\footnote{https://bats-tda.readthedocs.io/}
% \footnote{URL removed for anonymization}
will compute persistent homology of a Rips filtration and find the inverse map described in \cref{sec:opt_ph}. The bottleneck distance layer is supported by Hera \cite{jea_hera}, which can efficiently find the matching for bottleneck distance. 

\subsection{5-Tendril}
We generate a data set, 5-Tendril consisting of 500 points in 5 dimensions sampled along 5 tendrils, each of which consists of 50 randomly generated points on a canonical basis vector $e_i$ of Euclidean space. Here, our PH optimization method will first sample 100 points 
% (see \Cref{subsection:Sub-sampling}) 
and then optimize the bottleneck distance on H0 persistent diagram.
We perform 3 different dimension reduction algorithms: PCA, MDS and PH optimization. In \cref{fig:Tendril Result}, the results show that PCA and MDS can only see 3 branches, while PH optimization can see 5.

\subsection{Orthogonal Cycle}
We generate a data set of 500 samples in 5 dimensions with $\binom{5}{2} = 10$ cycles, each of which consists of 50 points and lies in a 2-dimensional plane spanned by two canonical bases $e_i$ and $e_j$ with center $e_i$ + $e_j$ and radius one. For speedup, our PH optimization method samples 200 points and then optimizes bottleneck distance on $H_1$. 

Since the dataset consists of cycles in orthogonal planes and our PH optimization method will also purse an orthonormal projection, there is no way to find a projection that can show and divide all cycles apart within the orthogonality constraint. In \Cref{fig:Cycle Result}, we show the dimension reduction results with 4 different methods: PCA will lead to a tangle where cycles cross with each other, and if without labels, one cannot tell how many cycles exist in this plot; ISOMAP will provide us a five-star with half of orthogonal cycles missed, but without labels, it is hard to determine if there are cycles exists; MDS can provide a 5-ring, but some of are not comprised of a single orthogonal cycle, for example, the yellow orthogonal cycle lies on two rings; our PH method can provide 3 cycles and the three straight lines between can indicate the collapse of orthogonal structure in the projection.

% \begin{figure*}[ht!]
%         \centering
%             \subfigure[PCA] 
%             {
%                 \label{subfig:cycle pca}
%                 \includegraphics[width=.24\textwidth]{images/cycle/PCA.pdf}
%             }
%             %
%             \subfigure[ISOMAP] 
%             {
%                 \label{subfig:cycle isomap}
%                 \includegraphics[width=.21\textwidth]{images/cycle/ISOMAP.pdf} % .png .jpg ... according to supported graphics files
%             }
%             %
%             \subfigure[PH] 
%             {
%                 \label{subfig:cycle PH}
%                 \includegraphics[width=.2\textwidth]{images/cycle/PH.pdf}
%             }
%             \subfigure[MDS] 
%             {
%                 \label{subfig:cycle MDS}
%                 \includegraphics[width=.215\textwidth]{images/cycle/MDS.pdf}
%             }
%             %
%         \caption{Dimension Reduction Results of Orthogonal cycles data set.  Left to right: PCA, MDS, ISOMAP and PH}
%         \label{fig:Cycle Result}
% \end{figure*}
\begin{figure*}[ht!]
        \centering
                \includegraphics[width=.24\textwidth]{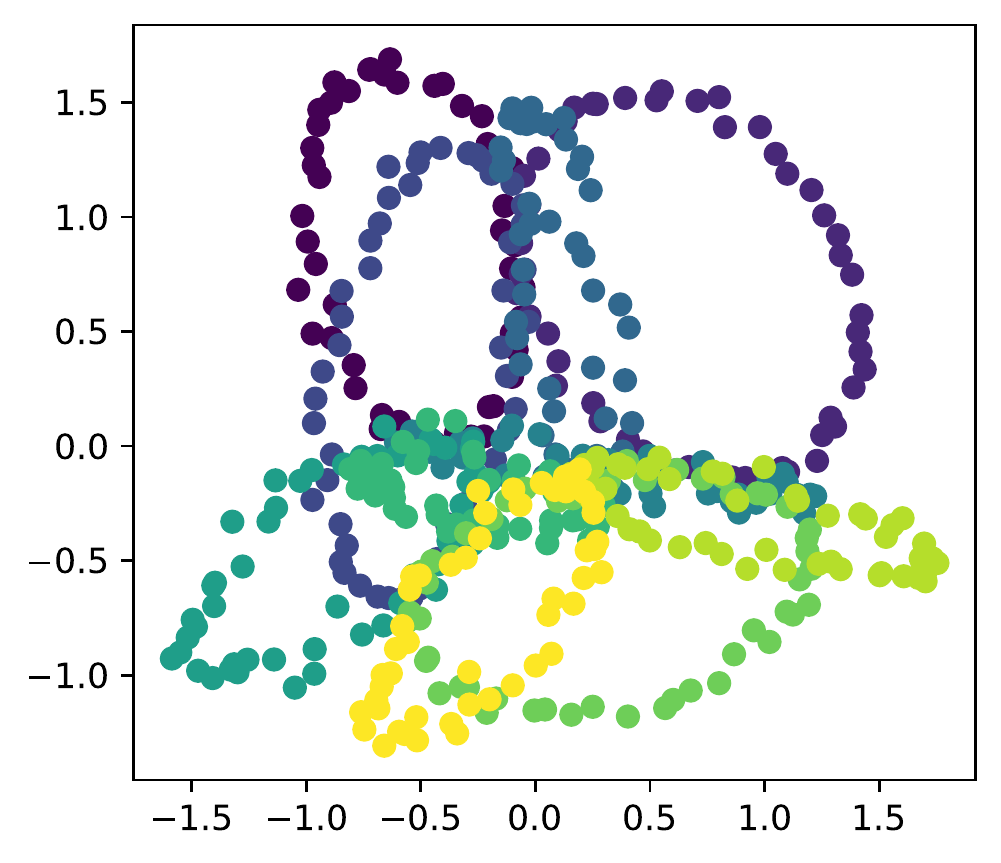}
                \includegraphics[width=.21\textwidth]{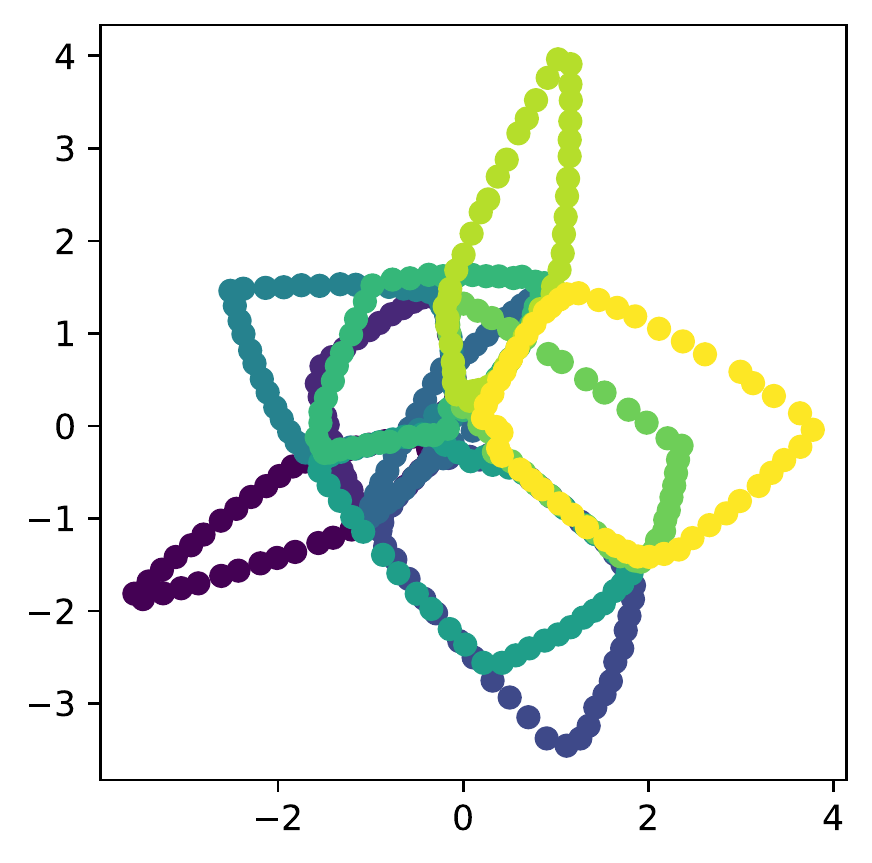} % .png .jpg ... according to supported graphics files
                \includegraphics[width=.2\textwidth]{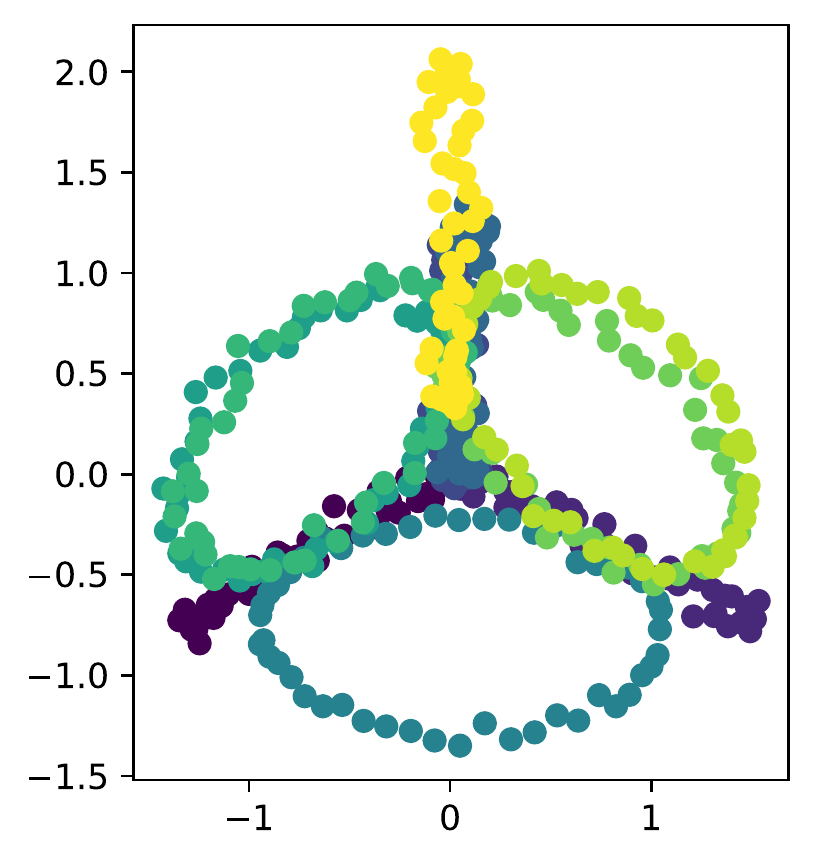}
                \includegraphics[width=.215\textwidth]{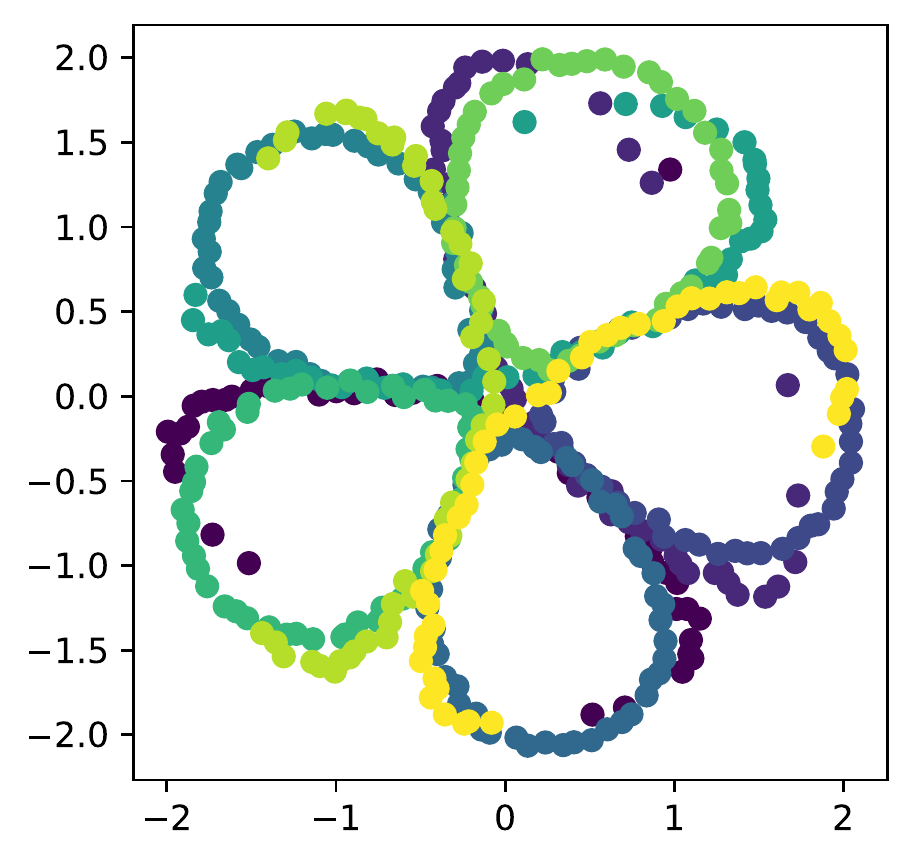}
        \caption{Dimension Reduction Results of Orthogonal cycles data set.  Left to right: PCA, MDS, ISOMAP and PH}
        \label{fig:Cycle Result}
\end{figure*}

\subsection{COIL-100}
The Columbia Object Image Library (COIL-100) \cite{Nene96objectimage} dataset contains 7200 colorful images of 100 objects, where each object has 72 $128\times128$ images with 3 color channels taken at pose intervals of 5 degrees. We pick the tomato dataset (see \Cref{fig:tomato}) with shape $72 \times 49152(= 128*128*3)$ from COIL-100 and perform our PH optimization algorithm on both $H_1$ and $H_0$ persistent diagrams. Although the number of points is only 72, the high feature dimension makes the direct running of our procedure on Pytorch prohibitive due to the memory issue of Cayley transform. To solve the problem, we devise two procedures: a) Optimizing the projection directly using the Cayley transform to handle orthogonality; b) first we perform PCA on the tomato dataset to reduce dimension to 10 and then keep reducing the dimension to 2 by our Pytorch PH optimization algorithm, denoted by PH+PCA. 

We include visualization of the results using methods PH, PH+PCA, PCA, ISOMAP, MDS into the Appendix and show the result of PH+PCA in \cref{fig:tomato_viz}. In \Cref{fig:tomato_viz}a), we can see that the 72 points in $\mathbb{R}^2$ constitute a great circle, which demonstrate that the dataset consists of pictures taken at 360 degrees at pose interval 5 around the tomato. In \Cref{fig:tomato_viz}b), we draw the persistence diagram of the projected data set and a unconfident band with width 9.065, which is the twice of bottleneck distance between the original and projected tomato dataset. By the discussion in \Cref{section: Measuring Embedding Distortion through Interleavings}, there is only one H1 class in the projected data that we can ensure also exists in the original one. In \Cref{fig:tomato_viz}c), we then draw the certain H1 representative in red and the longest uncertain H1 representative in blue.

\subsection{Natural Image Patches}

Natural image patches are a well-studied data set with interesting topological structures at various densities \cite{carlssonLocalBehaviorSpaces2008}.  We follow the data processing procedure of \cite{leeNonlinearStatisticsHighContrast2003} to sample $3\times 3$ patches from the van Hateren natural images database \cite{vanhaterenIndependentComponentFilters1998}.  We further refine a sub-sample of 50,000 patches using the co-density estimator of \cite{desilvaTopologicalEstimationUsing2004} with $k=5, p=40\%$ to obtain a data set of 20,000 patches which resembles the ``three-circle'' model of \cite{desilvaTopologicalEstimationUsing2004}.  

In \cref{fig:image_patches} we apply our procedure to two initial projections.  In both cases, we use a greedy subsampling of 100 points in the data which contains 5 robust $H_1$ pairs, agreeing with the three-circle model.  We first initialize the projection with the first two principal components of the data and then refine by optimizing the bottleneck distance on $H_1$ pairs.  The initial projection onto principal components displays a clear ``primary circle,'' with the two secondary circles projecting onto two chords.  Our procedure decreases the $H_1$ bottleneck distance on the 100 sampled points from $5.6$ to $4.7$ but there is minimal visual difference between the two projections, indicating that the initialization was near a local optimum.

\begin{figure}[h]
    \centering
    \includegraphics[width=0.45\linewidth]{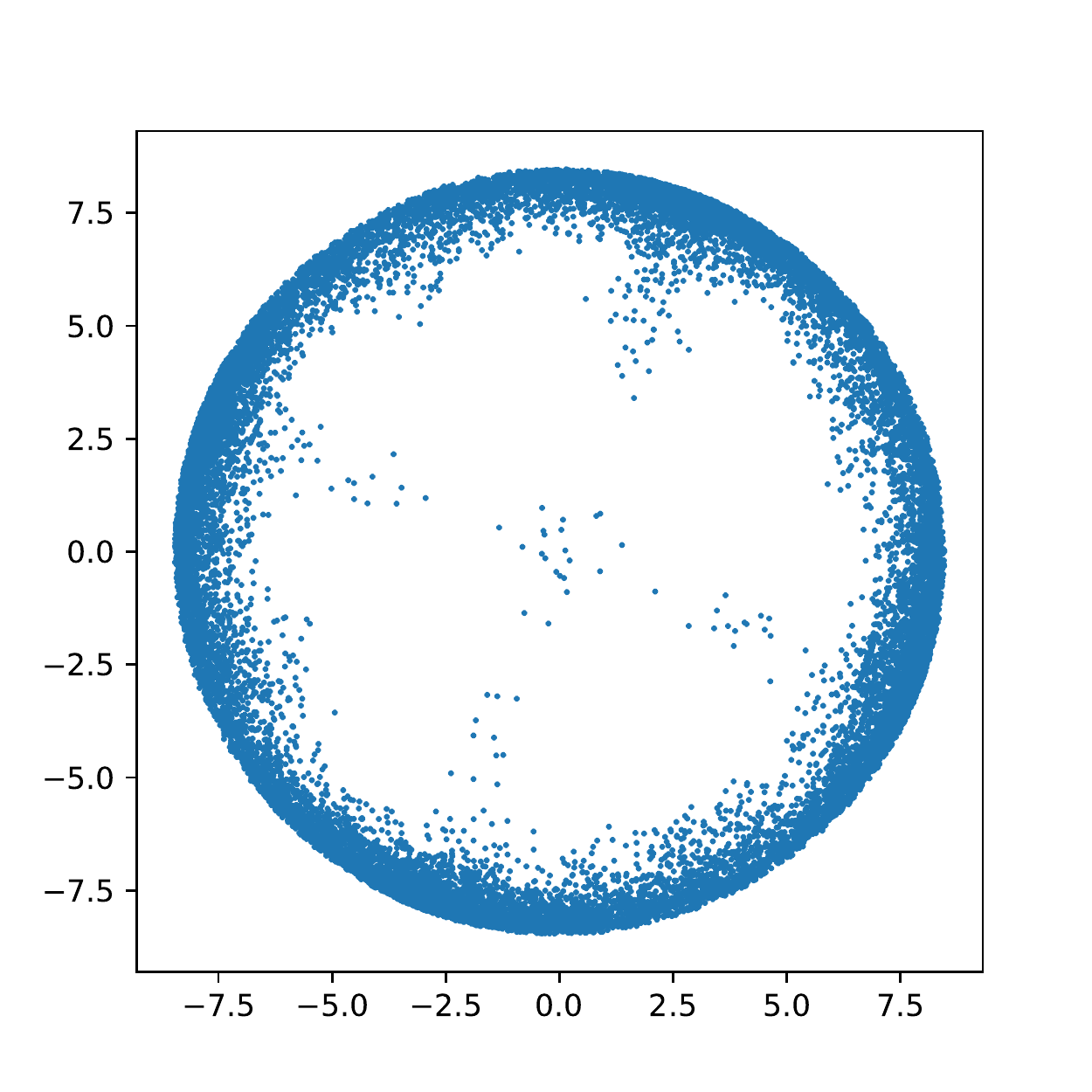}
    \includegraphics[width=0.45\linewidth]{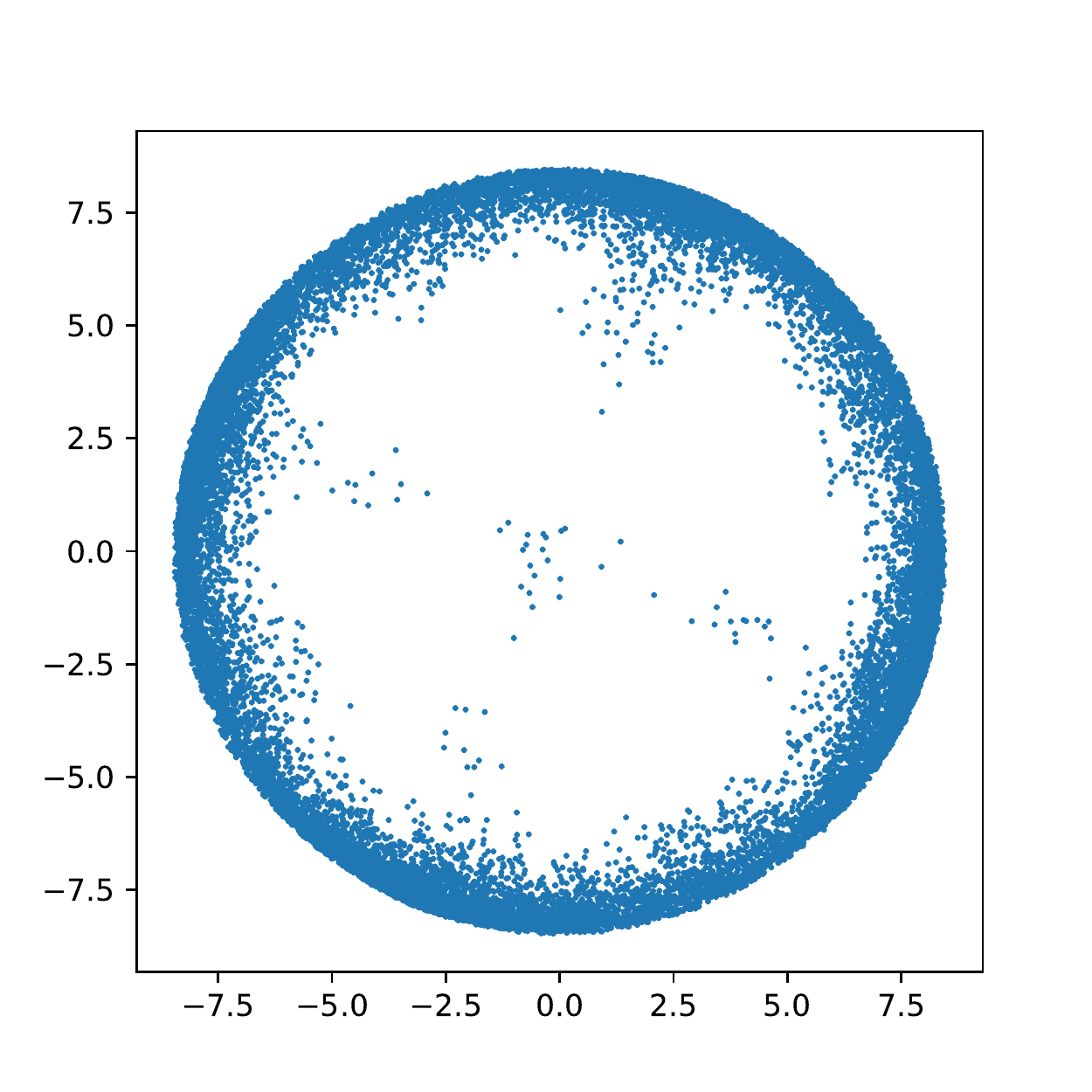}\\
    \includegraphics[width=0.45\linewidth]{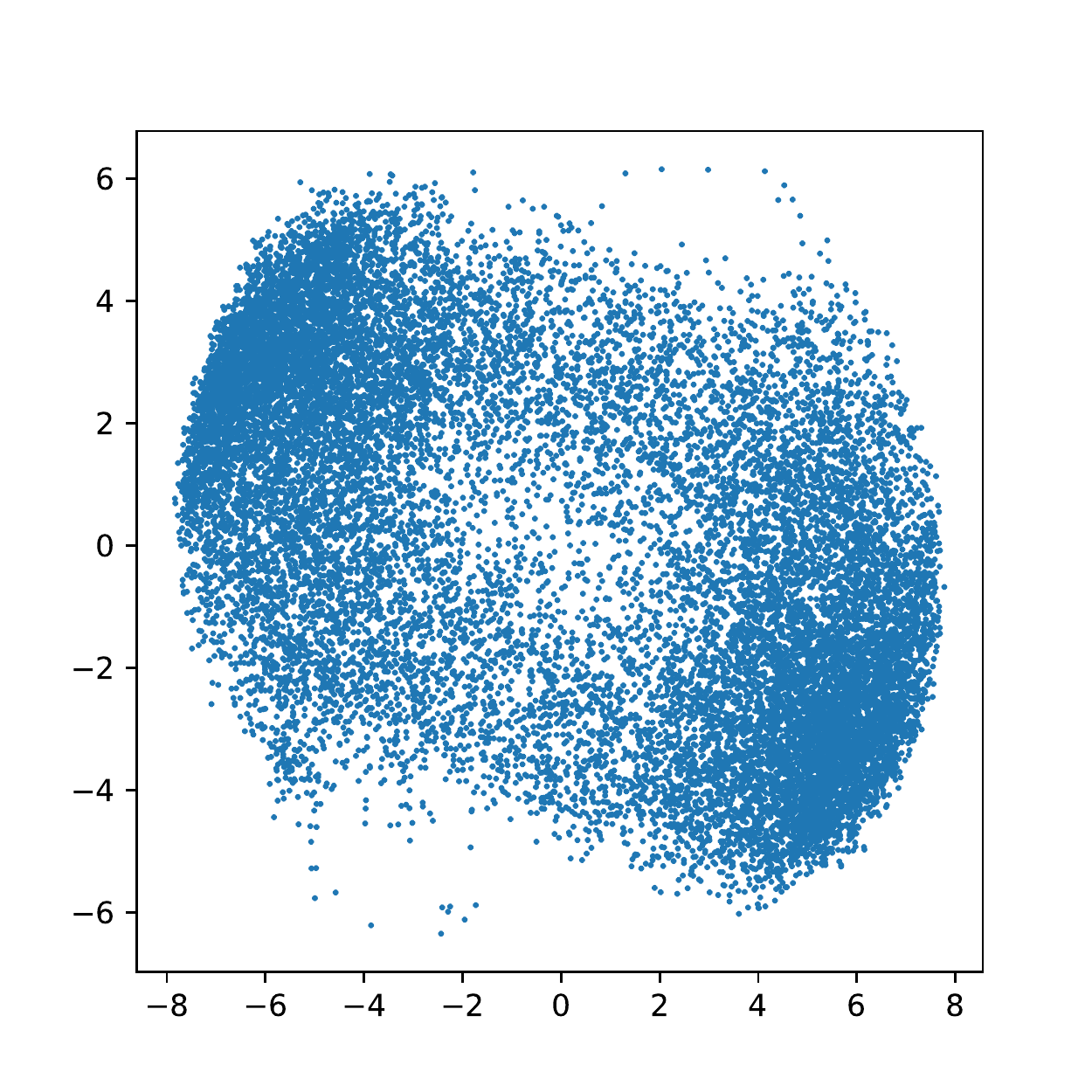}
    \includegraphics[width=0.45\linewidth]{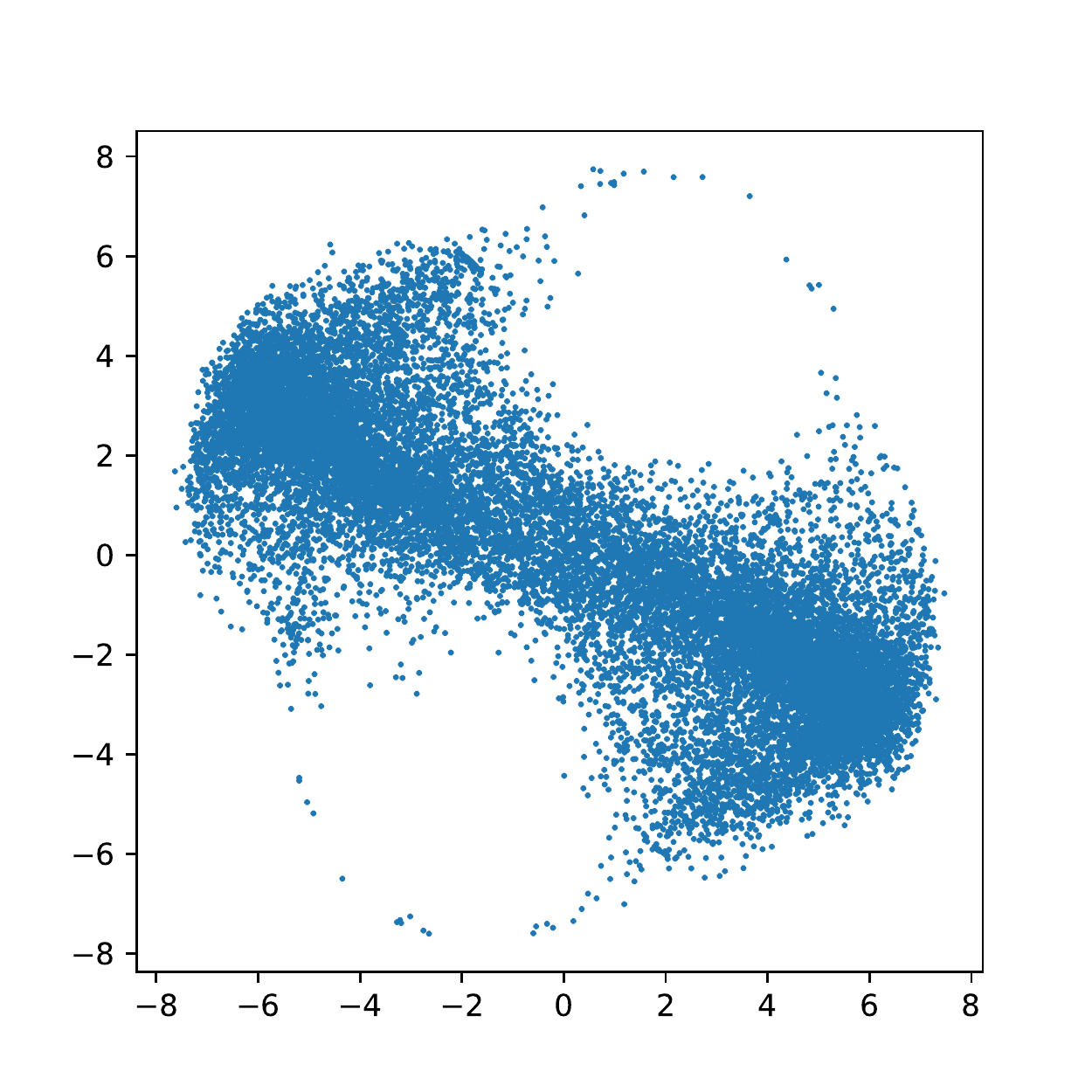}
    \caption{Projections of a $3\times 3$ image patch data set.  Top row, left to right: principal component embedding, and further refinement from minimizing $H_1$ bottleneck distance.  Bottom row: random projection and further refinement from minimizing $H_1$ bottleneck distance.}
    \vspace{-1.5ex}
    \label{fig:image_patches}
\end{figure}

The second row of \cref{fig:image_patches} starts with a random projection into two dimensions.  We again optimize to minimize the bottleneck distance on the $H_1$ pairs, and decrease the bottleneck distance on the sampled points from $7.9$ to $3.7$.  In this case, there is a noticeable visual difference between the two projections.  In the first projection, a noisy projection of the primary circle is visible, and in the second projection, we see a clear visualization of one of the two secondary circles, with the primary circle and other secondary circle collapsed to a chord.

In both these experiments, our $H_1$ bottleneck distance bounds do not allow us to confidently select any features in the visualization.  As with the orthogonal cycle data, the reason is fundamental: the three circle model embeds into a Klein bottle \cite{carlssonLocalBehaviorSpaces2008}, and any projection to fewer than 4-dimensions (necessary to embed the Klein bottle) will result in spurious intersections of at least two of the three circles.  Despite these limitations, our method is still able to present a subset of the important $H_1$ features of this data.

\section{Discussion}
% \textbf{Summarize what we did}
In this paper, we propose the use of the interleaving distance in dimensionality reduction.  We show that this distance can be used to identify topological features in correspondence between a full data set $X$ and a low dimensional embedding $Y$ using any dimension reduction procedure.  We also demonstrate how optimization of the equivalent bottleneck distance can increase the significance of important topologcial features in $X$ in the embedding $Y$.

We incorporate bottleneck distance optimization into projection pursuit and find that our method can preserve topological information when projecting from high dimensional spaces to two dimensions for visualization.  We find in several cases, our method will focus on visualizing a subset of the important topological structures as orthogonality of subspaces in the full data set prohibit the visualization of all structures using a single projection.  

Our method could be combined with other optimization objectives such as the maximization of variance in the projection as in PCA.  One limitation of our method is that the bottleneck distance is non-smooth and has many local minima.  Hybrid schemes which combine bottleneck distance optimization with other objectives may generally help with optimization.

One direction of future work that could help improve the ability to detect if topological features in embeddings has correspondence with the original data set would be to develop interleaving techniques based on non-linear shift maps.  Because the interleaving distance focuses on the worst possible distortion between persistent homologies, a more fine-grained analysis may reveal that more information is preserved in practice.

\section*{Acknowledgements:} BN was supported by the Defense Advanced Research Projects Agency (DARPA) under Agreement No.
HR00112190040. 

\bibliographystyle{acm}  
\bibliography{references}

\appendix
\onecolumn
\section{Appendix}
\subsection{Additional Dimension Reduction Results}
\begin{table}[ht] 
\centering 
\begin{tabular}{  >{\raggedright}m{2cm} m{5cm}  m{5cm} }      % centered columns (4 columns) 
\toprule                                   %inserts double horizontal lines 
Method & Bottleneck Distances (B0, B1) & Dimension reduced data \\  % inserts table heading 
\midrule\addlinespace[1.5ex]
        PH + PCA
        & \centerline{4.689, 4.532}
        & \includegraphics[height=1.7in]{images/tomato/tomato_PH.pdf} \\
\midrule\addlinespace[1.5ex]
        PH 
        & \centerline{5.295, 12.661}
        & \includegraphics[height=1.7in]{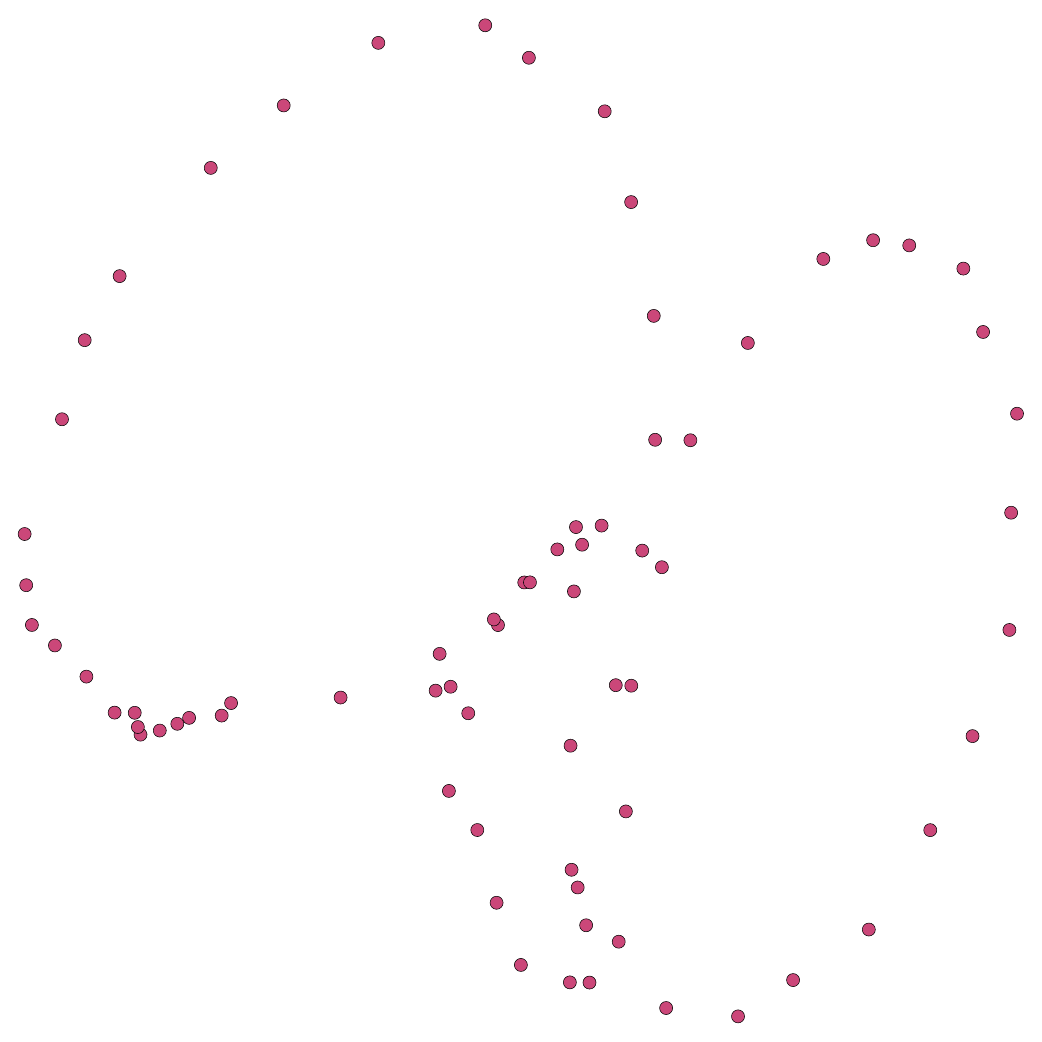} \\
\midrule\addlinespace[1.5ex]
        PCA
        & \centerline{5.148, 10.852}
        & \includegraphics[height=1.7in]{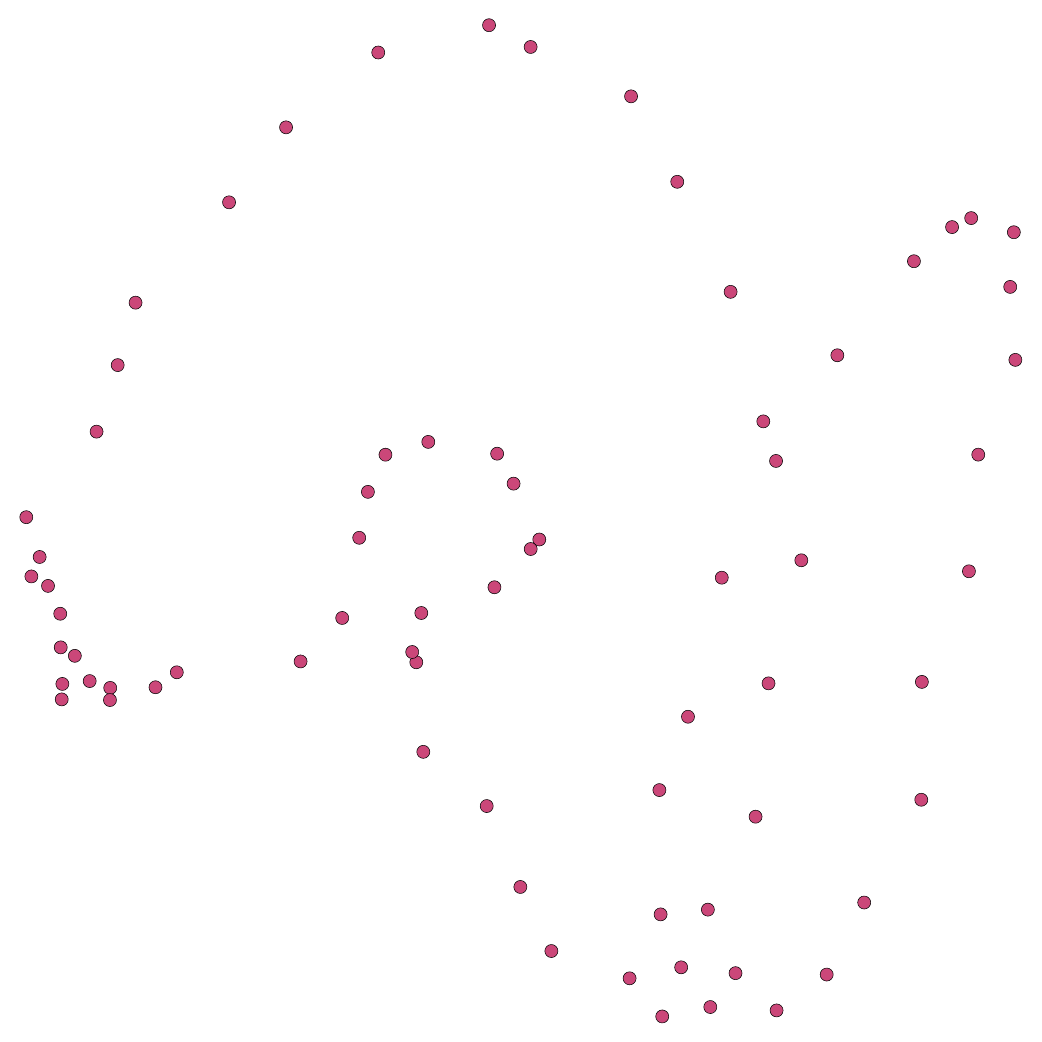}\\
\midrule\addlinespace[1.5ex]
        ISOMAP
        & \centerline{1.935, 108.623}
        & \includegraphics[height=1.7in]{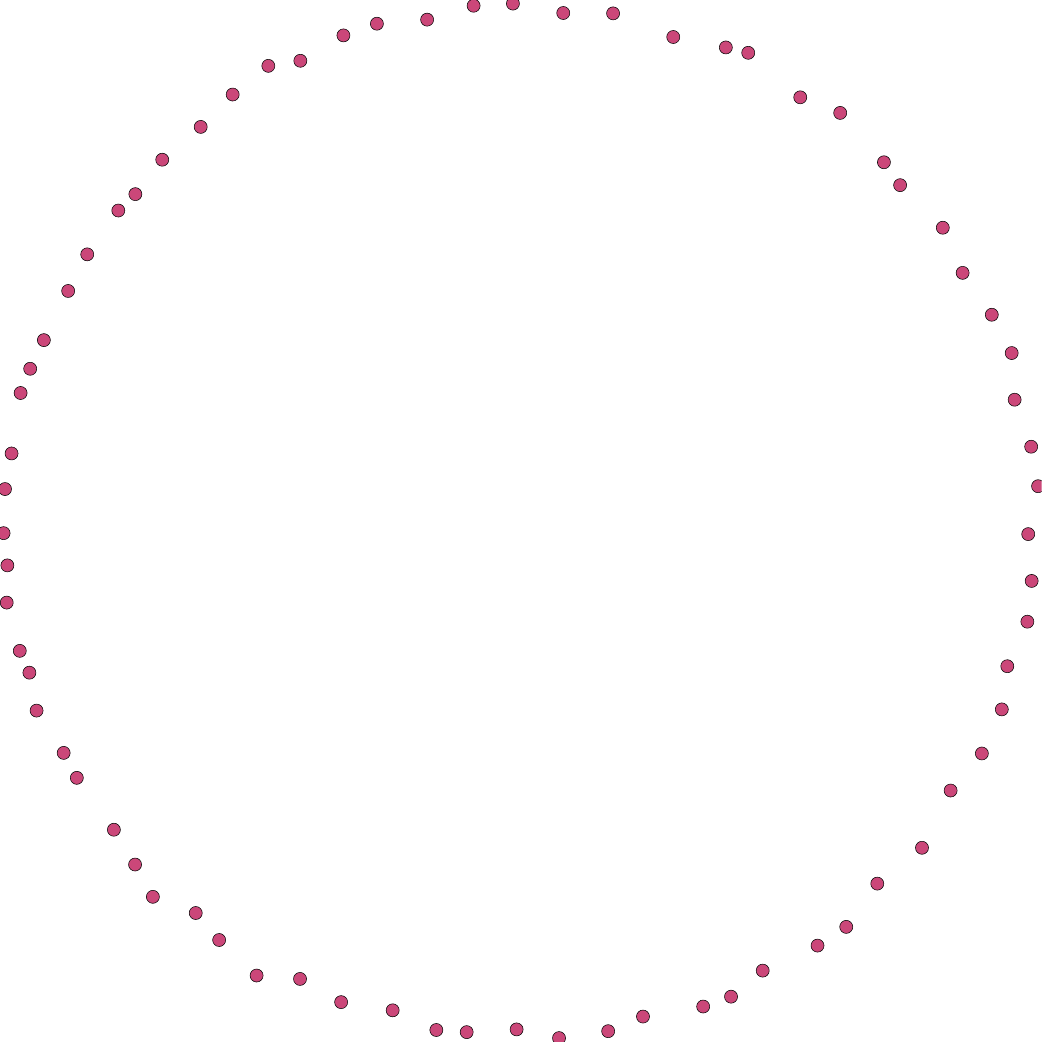} \\
\midrule\addlinespace[1.5ex]
        MDS
        & \centerline{4.782, 10.217}
        & \includegraphics[height=1.7in]{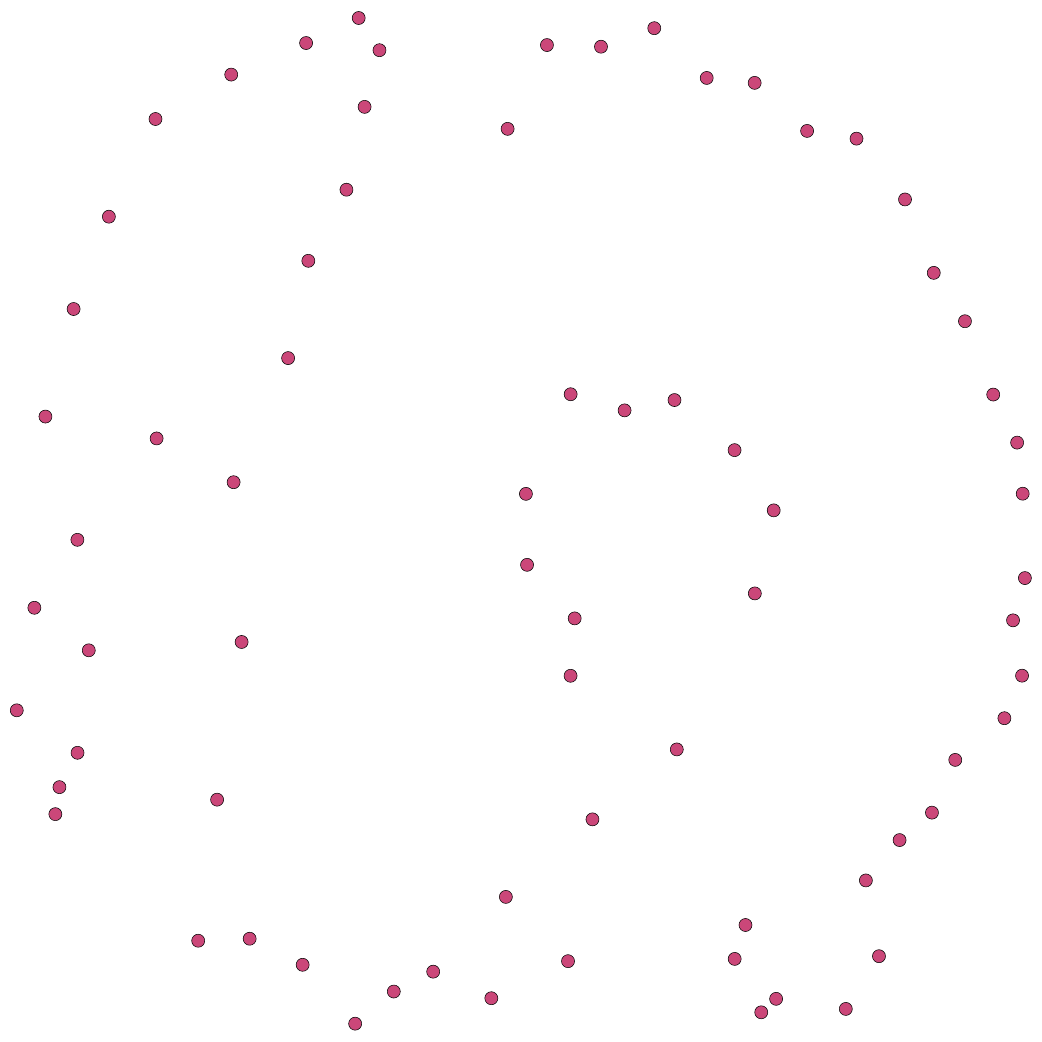} \\
\bottomrule
    \end{tabular}
    \caption{Dimension Reduction Results of tomato dataset from COIL-100}
    \label{table4.2}
\end{table}

\end{document}